\documentclass[twoside]{article}

\usepackage{arxiv_alt}

\usepackage{authblk}
\usepackage{xspace}
\usepackage{bibunits}
\usepackage{ulem}
\usepackage{hyperref}

\newcommand{\atlas}{\texttt{Atlas}\xspace}
\newcommand{\atlaslearn}{\texttt{Atlas-Learn}\xspace}
\newcommand{\atlasgrass}{\texttt{Atlas-Grass}\xspace}

\newtheorem{definition}{Definition}
\newcommand{\cM}{\mathcal{M}}

\begin{document}

%

%


\title{Atlas-based Manifold Representations\\ 
for Interpretable Riemannian Machine Learning}

\author[1]{Ryan A. Robinett}
\author[2]{Sophia A. Madejski}
\author[3,4]{Kyle Ruark}
\author[2,4,5]{\authorcr Samantha J. Riesenfeld}
\author[1,5]{Lorenzo Orecchia}
\affil[1]{Department of Computer Science, University of Chicago}
\affil[2]{Pritzker School of Molecular Engineering, University of Chicago}
\affil[3]{The College, University of Chicago}
\affil[4]{School of Engineering at Applied Sciences,  Harvard University}
\affil[5]{Department of Medicine, University of Chicago; CZ Biohub Chicago, LLC}
\affil[6]{NSF-Simons National Institute for Theory and Mathematics in Biology}

\maketitle






\begin{abstract}
Despite the popularity of the manifold hypothesis, current manifold-learning methods do not support machine learning
directly on the latent $d$-dimensional data manifold, as they primarily aim to perform dimensionality reduction into $\mathbb{R}^D$, losing key manifold features when the embedding dimension $D$ approaches $d$.
%
%
On the other hand, methods that directly learn the latent manifold as a differentiable atlas
have been relatively underexplored. 
%
%
In this paper, we aim to give a proof of concept of the effectiveness and potential of atlas-based methods.
To this end, we implement a generic data structure to maintain a differentiable atlas that enables Riemannian optimization over the manifold.
We complement this with an unsupervised heuristic that learns a differentiable atlas from point cloud data.
%
We experimentally demonstrate that this approach has advantages in terms of efficiency and accuracy in selected settings. 
Moreover, in a supervised classification task over the Klein bottle and in RNA velocity analysis of hematopoietic data, we showcase the improved interpretability and robustness of our approach.

\end{abstract}

\section{INTRODUCTION}
By extending concepts from flat Euclidean space to more complex, curved spaces, Riemannian manifolds provide a convenient mathematical model
to abstractly represent nonlinear data dependencies in high-dimensional datasets. 
The related manifold hypothesis~\cite{bengioRepresentationLearningReview2013}, which states that real-world data is predominantly concentrated on low-dimensional manifolds within a high-dimensional embedding space, has been both theoretically~\cite{carlsson} and experimentally~\cite{popeIntrinsicDimensionImages2021a, brownVerifyingUnionManifolds2022} validated for image data and is considered a key ingredient for the success of deep learning~\cite{bengioRepresentationLearningReview2013, bronsteinGeometricDeepLearning2021a}.
Similarly, natural scientists increasingly encounter empirical data that are well characterized by low-dimensional, often non-Euclidean, Riemannian manifolds.
For example, for several systems profiled by single-cell RNA-sequencing (scRNA-seq), which estimates the transcript counts of thousands of genes in tens (or hundreds) of thousands of individual cells, the data 
were shown to live near low-dimensional manifolds ($d \leq 5$) and display nontrivial geometry with statistically significant scalar curvature \cite{sritharan}.

As a result, there is now great interest in designing algorithms and simulating dynamics that eschew the high-dimensional ambient space and directly work with a representation of the low-dimensional manifold. 
Riemannian variants of generative diffusion models have recently been proposed~\cite{bortoliRiemannianScoreBasedGenerative2022, joGenerativeModelingManifolds2024, mangoubiEfficientDiffusionModels2025} and applied to protein generation~\cite{yimSE3DiffusionModel2023}. 
Stochastic models and dynamic optimal transport algorithms have been applied to learned manifold representations of transcriptomic spaces to study dynamic processes, such as cell-fate decisions in embryogenesis~\cite{Schiebinger2019-fq, moon, huguetManifoldInterpolatingOptimalTransport2022}.
%

Although Riemannian optimization techniques are well developed~\cite{absil}, most commonly employed manifold representation methods are ill-suited to their application, preventing the deployment of general-purpose machine learning techniques
to the low-dimensional manifolds underlying much real-world data.
Indeed, the steps of even the simplest Riemannian first-order methods are applications of the manifold's exponential map (or more generally of a retraction)~\cite{hosseini_and_sra}, which require knowledge of the manifold's tangent planes and differentiable structure.
%
By contrast, current manifold learning techniques do not attempt to learn representations that constrain downstream tasks to work within the manifold and its intrinsic geometry. 
Rather, they focus heavily
on dimensionality reduction, i.e., constructing embeddings of data that are naturally represented in a high-dimensional space $\mathbb{R}^n$ into a lower-dimensional space $\mathbb{R}^D$ , where $D << n$ is still higher (often substantially) than the intrinsic dimension of the latent data manifold~\cite{meilaManifoldLearningWhat2023}.
%
Such methods often aim to preserve fundamental intrinsic topological invariants (such as homology) and geometric invariants (such as geodesic distances), but not support optimization. 
For instance, even the simple task of evolving a given ODE by forward-Euler discretization over the manifold cannot be realized in a generic dimensionality-reduced embedding, as the evolving point may leave the manifold and travel freely in the $D$-dimensional ambient space. 
%
This issue arises in many scientific applications where experimental data naturally include a velocity vector field over the manifold, such as RNA velocity~\cite{frank_rec, qiuMappingTranscriptomicVector2022, wang_velocity_2025, topicvelo}. 
%
%
See Section~\ref{sec:related} for further discussion of 
different manifold learning methods.


%

\subsection{Our Contribution}

The previous discussion immediately suggests considering manifold representations that directly parametrize the latent manifold via diffeomorphic charts.
Because most real-world latent manifolds are not diffeomorphic to Euclidean space~\cite{cooley}, this approach explicitly requires the use of multiple charts and differentiable transition maps between overlapping charts, i.e., the maintenance of a differentiable atlas~\cite{lovettDifferentialGeometryManifolds2010}.
%
%
%
While some such methods have been proposed,
this approach has been relatively underexplored and has found limited practical application. 
%
A common objection to the deployment of this idea
are the tedious and computationally expensive numerical operations required to construct and maintain a large number of charts and overlaps. 

The main goal of this paper is to provide a proof of concept demonstrating that practical methods that explicitly learn and maintain a differentiable, approximate atlas can enable scalable
analysis of low-dimensional manifold structure that is more faithful and interpretable than afforded by current methods. 
%
%
To this end, we implement a straightforward such data structure, termed \atlas, which, while not optimized, allows for fast chart-membership queries and numerical approximation of exponential maps.
%
Our implementation can represent both known atlases of algebraic manifolds and atlases learned from empirical manifolds given as high-dimensional point-cloud data.
For the latter, we also present a simple heuristic, \atlaslearn, for constructing an \atlas representation via \emph{approximate quadratic coordinate charts}~\cite{sritharan}.
%
%
The rest of the paper is dedicated to exploring the \atlas properties. Our empirical analyses showcase its efficiency and its capability to recover manifold structure and enable machine-learning routines to work directly in the inherent geometry of the data.
In particular, we demonstrate that our atlas scheme can both speed up first-order Riemannian optimization over manifolds with closed-form algebraic structure and enable it in the first place over manifolds learned from empirical point cloud data. 
The following are our main contributions in this direction:
\begin{enumerate}
\item In Section~\ref{sec:grassmann}, we show that the \atlas scheme can speed up first-order Riemannian optimization for a benchmark online subspace learning problem on the Grassmann manifold, a classical target of Riemannian optimization routines~\cite{absil}. 
%
\item In Section~\ref{sec:comparison}, we compare \atlaslearn with existing manifold-learning techniques on the task of reconstructing an established manifold parametrization of Carlsson's high-contrast, natural image patches~\cite{carlsson, sritharan} from the point cloud data. We show that our scheme outperforms competitors in preserving homology
and approximate geodesic path lengths.

\item In Section~\ref{sec:classification_patches}, we implement a Riemannian optimization algorithm -- the Riemannian principal boundary (\texttt{RPB}) algorithm~\cite{yao_2020} -- to solve a classification problem over the Klein bottle, a manifold with nontrivial topology, using \atlas-based primitives. The results demonstrate that, relative to existing manifold learning methods, the atlas approach has better accuracy and interpretability.
\item Finally, in Section~\ref{sec:bio}, we apply \atlaslearn and \atlas to the reconstruction of transcriptomic dynamics over a well-studied hematopoietic scRNA-seq dataset~\cite{qiuMappingTranscriptomicVector2022} and find that the predicted flows adhere much more closely to the latent manifold than do those of previous approaches based on higher-dimensional embeddings.
\end{enumerate}

As prototype solutions to the challenge of learning an effective atlas-like manifold representation, both \atlas and \atlaslearn can likely be improved upon significantly. Indeed, the main point of our paper is \emph{not} their algorithmic optimization, but the finding that such a simple framework already improves upon the inadequacies of existing manifold learning methods.
We hope that our paper can inspire others to design improved algorithms and  use our approach to verify empirical results obtained by other methods.
%

%
%

\subsection{Related Work} \label{sec:related}

\paragraph{Manifold Learning}

Previous work on manifold learning has been centered on the problem of reducing the dimension $n$ of the ambient space $\mathbb{R}^n$ to a target dimension $D$, while preserving the structure of the $d$-dimensional latent manifold, with the goal of saving space and processing time~\cite{meilaManifoldLearningWhat2023}.
The earliest example of such methods is principal component analysis (\texttt{PCA})~\cite{PrincipalComponentAnalysis2002}, which simply projects onto the $D$-dimensional subspace preserving the most variance in the data.
\texttt{Isomap}~\cite{tenenbaumGlobalGeometricFramework2000} and its variations~\cite{hongyuanzhaIsometricEmbeddingContinuum, weinbergerUnsupervisedLearningImage2006} attempt to construct \emph{metric embeddings} of the latent manifold into $\mathbb{R}^d,$ i.e., embeddings of the data in which the Euclidean geodesic (straight-line) distances in $\mathbb{R}^D$ equal (or approximate) the manifold geodesic distances.
Clearly, this is only possible for flat manifolds, no matter how large $D$ is chosen\footnote{Isometric embeddings, such as those of Nash~\cite{nashC1IsometricImbeddings1954,nashImbeddingProblemRiemannian1956}, preserve the Riemannian metric and the geodesic distance \emph{along the manifold}, not with respect to the metric structure of the ambient space.}\cite{robinsonSphereNotFlat2006}.
Diffusion maps~\cite{coifmanDiffusionMaps2006} and Laplacian eigenmaps~\cite{belkinLaplacianEigenmapsSpectral2001a, belkin_2006} successfully approximate the Laplace-Beltrami operator of the latent manifold and yield embeddings that capture the diffusion distance along the manifold. However, they do not enable the execution of optimization routines over the manifold.
%
More advanced methods learn~\cite{koelleManifoldCoordinatesPhysical2022, donohoHessianEigenmapsLocally2003, roweisNonlinearDimensionalityReduction2000} local representations, akin to charts, but combine them into one global Euclidean representation, rather than constructing an atlas.
Finally, methods like $t$-SNE and UMAP, mostly used for visualization are known in practice to greatly distort topological information, particularly in the case of scRNA-seq data~\cite{cooley}.

Ultimately, all these methods aim to produce  a faithful embedding of the latent manifold into $\mathbb{R}^D$, 
which is known to exist for $D \geq 2d$ by Whitney's celebrated embedding result~\cite{whitney1936differentiable}. However, even Whitney's embedding, which is the gold standard of manifold learning~\cite{meilaManifoldLearningWhat2023}, would not let us work directly on the latent manifold. 
To address this issue, all the above methods are often run with $D \approx d$ in the attempt to construct a global single chart for the data manifold. However, this is liable to cause significant distortions in the resulting embedding for manifold with non-trivial topology, which are not homeomorphic to $\mathbb{R}^D.$

There are, however, some good news in that many crucial manifold invariants can be accurately estimated from samples in a local fashion. For instance, 
Little \textit{et al.} prove that if a manifold $\mathcal{M}\subset\mathbb{R}^D$ is observed as a finite sample perturbed by sub-Gaussian noise, then multiscale singular value decomposition (\texttt{mSVD}) can be used to compute both local tangent plane approximations and the intrinsic dimensionality of $\mathcal{M}$ \cite{little}.
In more empirical work,
Sritharan \textit{et al.} describe a method to estimate the Riemannian curvature tensor via  the Gauss-Codazzi equations by computing local quadratic approximations to the manifold  \cite{kac,sritharan}. 
These successful local algorithms, \texttt{mSVD} and local quadratic approximations, will also be the main workhorses of our approach.

\paragraph{Riemannian Optimization}

While many Riemannian optimization routines are intuitive adaptations of Euclidean first- and second-order methods~\cite{absil, hosseini_and_sra}. However, it is far more complex to compute analogous first- and second-order primitives on Riemannian manifolds. For example, the action of a gradient on a point in Euclidean space is simply vector addition, whereas in a Riemannian manifold, this action requires the computation of the exponential map, which is the solution to a specific second-order ordinary differential equation; while closed-form solutions are known for some algebraic manifolds~\cite{absil, boumal_2016}, they are generally unavailable, even for cases as simple as oblate ellipsoids in $\mathbb{R}^3$ \cite{ganshin_1969}.
Circumnavigating this difficulty is often discussed in terms of retractions, i.e., smooth maps that assign to each point a map from the tangent space to the manifold which agrees with the exponential map in the zero-th and first derivatives. 
Euclidean ``step-and-project'' methods, which make updates on a manifold using Euclidean gradients in an ambient space before mapping the update back into the manifold, are prototypical examples of retractions, but can suffer from slow running time due to numerical condition issues \cite{absil}.
Because finding sufficiently fast, accurate retractions is a challenge, they are often discussed in theory  rather than practice \cite{hosseini_and_sra}, with the notable exception of the B\"urer-Monteiro scheme for convex optimization on the  PSD elliptope \cite{cifuentes_2021}.

%

For general manifolds observed empirically as point clouds, 
existing Riemannian optimization techniques have seen limited application, as the differential-geometric primitives lack known closed forms.
In these cases, Euclidean optimization methods are usually deployed on
low-dimensional Euclidean representations learned via the dimensionality reduction techniques discussed above~\cite[e.g.]{huguetManifoldInterpolatingOptimalTransport2022}.  However, downstream machine learning tasks, such as clustering and inference of pseudo-time, have been shown to be sensitive to the choice of low-dimensional representation~\cite{patruno_2020}, emphasizing the need for a more robust approach to Riemannian optimization over learned manifolds.

\section{Methods}
This section summarizes our methodology, including the specification of the \atlas data structure and motivation for its design (Sec.~\ref{subsec:data-structure}), the use of the \atlas to quickly compute retractions, i.e., approximations to the exponential map, and vector transports via simple Euclidean updates (Sec.~\ref{subsec:retractions}), and the method \atlaslearn for constructing an \atlas data structure from point-cloud data (Sec.~\ref{subsec:atlaslearn}).
%

\subsection{The \atlas Data Structure}\label{subsec:data-structure}

\begin{figure}[h]
    \centering
    \begin{overpic}[width=0.49\columnwidth, trim=35 30 40 20, clip]{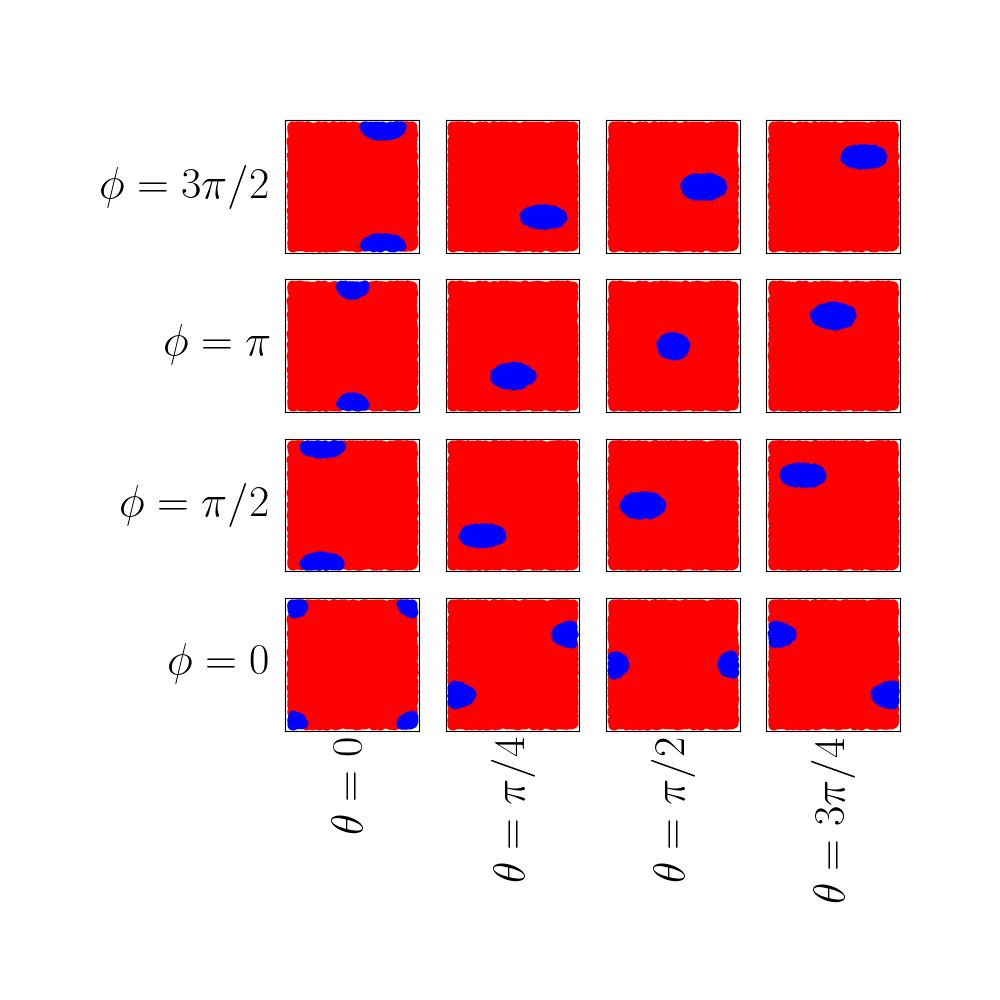} \put(5,95){\textbf{A}} \end{overpic}
    \begin{overpic}[width=0.49\columnwidth, trim=35 30 40 20, clip]{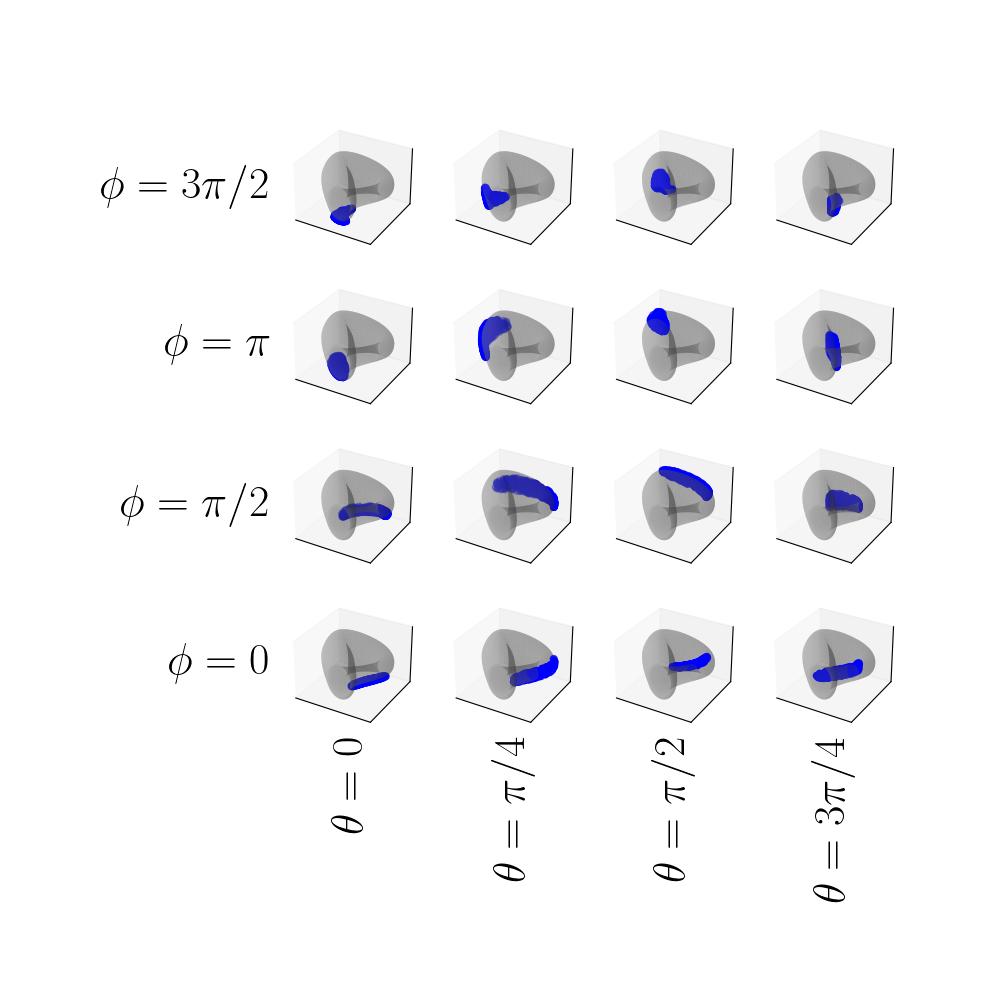}\put(5,95){\textbf{B}}
    \end{overpic}

    \caption{\textbf{Chart domains for a differentiable atlas of the Klein bottle} \textbf{A.} Charts depicted in polar coordinates show points belonging (blue) or not (red) to the chart, and indexed by the $(\theta, \phi)$ coordinates of the chart center $k(\theta,\phi)$ (Sec.~\ref{sec:comparison}). \textbf{B.} The same coordinate charts shown in the Karcher representation.}
    \label{fig:klein_ag_example}
\end{figure}

\paragraph{Differentiable Atlases} Since the work of Whitney, the notion of a differentiable atlas has been a crucial component of the modern mathematical definition of a smooth manifold~\cite{lovettDifferentialGeometryManifolds2010,kobayashi_nomizu}.
%
%
A differentiable atlas for a manifold $\mathcal{M}$ of dimension $d$ is a collection $\left\{\varphi_i: \mathcal{V}_i \subseteq \mathbb{R}^d  \to \mathcal{U}_i\right\}_i$ of bijective coordinate charts\footnote{The coordinate charts are also homeomorphisms. For simplicity, in this presentation, we omit standard topological details.} 
 whose image covers $\mathcal{M}$, i.e., $\bigcup_i \mathcal{U}_i=\mathcal{M}$. 
For coordinate charts $\varphi_i$ and $\varphi_j$ whose images overlap, i.e., $\mathcal{U}_i \bigcap \mathcal{U}_j \neq \emptyset$, an atlas also specifies bijective \textit{transition maps}
\begin{equation}\label{eqn:transition_maps}
    \psi_{ij}:=\left.\left(\varphi_j^{-1} \circ \varphi_i\right)\right\rvert_{\varphi_i^{-1}(\mathcal{U}_i \cap \mathcal{U}_j)}.
\end{equation}
The atlas is differentiable if all the transition maps are differentiable. This definition neatly allows $\mathcal{M}$ to be an abstract topological space, by shifting all the differentiability requirements to the transition maps.
A collection of bijective differentiable transition maps uniquely identifies a differentiable manifold, up to diffeomorphism, even in the absence of corresponding charts.
Motivated by this property,  we will remove the requirement of Equation~\ref{eqn:transition_maps} and define transition maps separately from coordinate chart images. This will significantly simplify the computational work needed to handle empirical manifolds. 

\paragraph{The \atlas Data Structure}
We assume, without loss of generality, that $\mathcal{M}$ is presented through an embedding in high-dimensional space $\mathbb{R}^D$. For applications to point-cloud data, this embedding may have been preprocessed by \texttt{PCA}-based dimensionality reduction to accelerate computation.
An instance of the \atlas data structure maintains $n$ coordinate charts of the form $\{\varphi_i: \mathcal{V}_i \subseteq \mathbb{R}^d \to \mathbb{R}^D\}_{i \in [n]}$, which are differentiable injections. An undirected atlas graph $G=([n], E)$ describes which charts overlap, i.e., $\{i,j\} \in E$ if charts $i$ and $j$ overlap. In this case, a differentiable injective transition map $\psi_{ij}$ is defined from a subset $\mathcal{V}_{ij} \subset \mathcal{V}_i$ to $\mathcal{V}_j.$
The basic attributes and methods of our data structure, are given in Algorithm~\ref{spec:atlas_graph}. 

\paragraph{Patching Together Local Charts} 
The definition of \atlas is motivated by the desire to accommodate not only atlases for manifolds with known structure, but also atlases learned from point-cloud data, which may not exactly conform to the definition of a differentiable atlas. To understand this challenge, let us briefly consider the behavior of \atlaslearn (further detailed in Sec.~\ref{subsec:atlaslearn}).
\atlaslearn starts by locally learning tangent planes via PCA and defining charts in tangential coordinates by performing quadratic regression in the normal space. An immediate obstacle is 
that different charts are solving separate regression problems: for example, it may happen that no pair of images $\phi_i(\mathcal{V}_i)$ and $\phi_j(\mathcal{V}_j)$ even overlap, 
so that transition maps of the form of Equation~\ref{eqn:transition_maps} become undefined.
To accommodate this possibility, while still being able to work on the manifold, we weaken the requirement imposed on transition maps,
allowing the \atlas data structure to build arbitrary differentiable injective transition maps between charts domains, possibly failing to obey Equation~\ref{eqn:transition_maps}.
In other words, we will use the coordinate charts just as canonical maps from each domain $\mathcal{V}_i$ to the embedding space $\mathbb{R}^D$, while the global differentiable structure will be determined by the transition maps.

\begin{algorithm*}[t]
    \caption{\atlas \\
    Generic specification for \atlas data structure} \label{spec:atlas_graph}
    \begin{algorithmic}
        \Require target manifold dimension $d$, ambient dimension $D$, number $n$ of charts, atlas graph $G=([n],E)$
        
        \ReqDo
        
        \For{$i \in [n]$}
            \State $\mathtt{in\_domain}_i: \vec{\xi} \in \mathbb{R}^d \to \{\mathbb{T}, \mathbb{F}\} $ 
            \Comment{Implicitly defines the domain $\mathcal{V}_i$ for chart $\phi_i$}

            \State $\texttt{identify\_new\_chart}_i: \vec{\xi} \in \mathcal{V}_{i} \to 2^{[n]}$ 
            \Comment{Identifies the charts $j$ for which $\psi_{ij}(\vec{\xi})$ is defined, defining $\mathcal{V}_{ij}$} 
            
            \State $\mathtt{coord\_chart}_i: \vec{\xi} \in \mathcal{V}_i \to 
            \mathbb{R}^D$ 
            \Comment{Coordinate chart mapping $\vec{\xi}$ to ambient coordinates $\varphi(\vec{\xi})$}

            \State $\texttt{D\_chart}_i: \vec{\xi} \in \mathcal{V}_i \to \mathbb{R}^{D \times   d}$
            \Comment{Maps $\vec{\xi}$ to the differential operator $D\varphi_i(\vec{\xi})$}


            
            
            
            
        \EndFor

        \For{$\{i,j\} \in E$}
            \State $\psi_{ij}: \mathcal{V}_{ij} \to \mathcal{V}_j$ \Comment{Transition map from $\mathcal{V}_{ij}$ to $\mathcal{V}_j$}
            \State $\texttt{transition\_vector}_{ij}: (\vec{\xi}, \vec{\tau}) \in  \mathcal{V}_{ij} \times \mathbb{R}^d \to \mathbb{R}^d$ \Comment{Maps tangent vector $\vec{\tau}$ at $\vec{\xi} \in \mathcal{V}_i$ to $[D \psi_{ij}]_{\vec{\xi}}(\vec{\tau})$}
        \EndFor
        
        \ReDo
    \end{algorithmic}
\end{algorithm*}

%


\paragraph{Discrepancy Between Charts}
Naturally, our \atlaslearn heuristic will enforce a degree of compatibility between the coordinate charts $\varphi_i$ and the transition maps. In particular, \atlaslearn will ensure that, for any point $\vec{\xi} \in \mathcal{V}_{ij}$, the discrepancy between coordinate charts and transition maps
$$
\| \varphi_i(\vec{\xi}) - \varphi_j(\psi_{ij}(\vec{\xi}))) \|_{\mathbb{R}^D}
$$
goes to $0$ as the number of charts $n$ and the  number of points sampled from a smooth goes to infinity. Notice that if the discrepancy is $0$ for all points $\vec{\xi}$, Equation~\ref{eqn:transition_maps} holds and  the \atlas data structure actually yields a differentiable atlas (Section~\ref{app:atlas}).

In contrast, when a discrepancy exists, error may be introduced that can accumulate over multiple transitions. The closer $\psi_{ji}$ is to $\psi_{ij}^{-1}$, and the fewer times a transition boundary needs to be invoked, the less error is introduced.
Empirically, our results demonstrate numerical robustness of \atlas-enabled computations on various problems despite error introduced from transition maps.


 %
 %
 %


\subsection{Building Retractions via \atlas}\label{subsec:retractions}

In Riemannian optimization, hard-to-compute exponential maps are often replaced by retractions~\cite[e.g.,]{hosseini_and_sra}. 
A retraction $R_{\vec{p}}$ at a point $\vec{p} \in \mathcal{M}$ is a local diffeomorphism from the tangent plane $\mathcal{T}_{\vec{p}}\mathcal{M}$ at $\vec{p}$ to $\mathcal{M}$ which maps $\vec{0}$ to $\vec{p}$ and has the identity as its differential at $\vec{0}$. 
This ensures that the retraction is a first-order approximation to the true exponential map.

Given an \atlas data structure, we can exploit our coordinate charts to give an efficient \emph{local} retraction by the following simple construction~\cite{absil}. Let $\vec{p} =\mathtt{coord\_chart}_i(\vec{\xi})$ and define:
\begin{equation}\label{eq:retraction}
R_{\vec{p}}(\vec{\tau}) := \mathtt{coord\_chart}_i(\vec{\xi} + \mathtt{D\_chart}_i(\vec{\xi})^{-1} \vec{\tau})
\end{equation}
Elementary differentiation shows that $R_{\vec{p}}(\tau)$ is a valid retraction at $\vec{p}$ as long as 
$\vec{\xi} + \mathtt{D\_chart}_i(\vec{\xi})^{-1} \vec{\tau} \in \mathcal{V}_i.$
When this is not the case, we are anyway forced to transition to a different chart. When the coordinate is barely past the transition boundary, we simply transition charts after the update\footnote{In our specific experiments, due to the relatively small magnitude of our first-order updates, we only have to consider this case.}. Otherwise, we perform the update by truncating $\vec{\tau}$ to $\vec{\tau}' = c \cdot \vec{\tau}$, for the largest $c$ such that $\vec{q} = R_{\vec{p}}(\vec{\tau}')$ still lands in $\mathcal{V}_i.$ We then transition $\vec{q}$ and the tangent vector $(1-c) \cdot \tau$ to the next chart.
In the rest of the paper, 
we refer to updates performed via the retraction of Equation~\ref{eq:retraction} as \textit{quasi-Euclidean update}, as, in the coordinate chart representation, they simply add a linear transformation of the update direction $\vec{\tau}$ to the current point $\vec{\xi}.$
In our implementation, we rely on the basic primitives of the \atlas data structure to implement other fundamental manifold methods, including retraction-logarithms (the inverses of retractions), geodesic computations, and vector transports. We provide details 
in the Supplementary Materials (Sec.~\ref{app:approx_primitives}).

\subsection{The \atlaslearn Heuristic} \label{subsec:atlaslearn}

We assume that the dimensionality of the manifold representation be specified \textit{a priori}; this can be learned using dimensionality estimation methods, e.g., multiscale singular value decomposition (mSVD)~\cite{little}.
The \atlaslearn starts 
by using $n$-medoids clustering~\cite{schubertFastEagerKmedoids2021} to partition the $N$-point cloud $X \in \mathbb{R}^{N \times D}$ into $n$ subsets $X_1, \ldots, X_n$, which will correspond to the images of the coordinate charts.
For each $X_i$, we compute the mean $\vec{m}_{i}$ of points in $X_i$ and learn a local tangent plane by computing the top $d$ principal components of the data $X_i$ centered at $m_i$. 
This yields a Stiefel matrix $L_i \in \mathbb{R}^{d \times D}$, which projects onto the tangent space passing through $m_i$, and another Stiefel matrix $M_i \in \mathbb{R}^{(D-d) \times D}$, which projects onto the orthogonal normal space.

With this representation of the tangent space in hand, we parametrize the manifold over $X_i$ by performing a least-squares quadratic regression from the tangent coordinates $L_i^\top X_i$ to the normal coordinates $M_i^\top X_i$.
Let $f_i : \mathbb{R}^d \mapsto \mathbb{R}^{D-d}$ be the resulting quadratic function.
Then, we can implement part of specification of \atlas by the standard parametrization of a graph manifold~\cite{docarmo_2016}, where
$\mathtt{coord\_chart\_i}(\vec{\xi}) :=(\vec{\xi},f_i(\vec{\xi}))$,
and $\mathtt{D\_chart}_i$ is just its differential.
The transition map is constructed as:
$$
\Psi_{ij}(\vec{\xi}) =  L_j^\top \left[ \left(L_i \vec{\xi} + \vec{m}_{i}\right) - \vec{m}_{ j}\right]
$$
with $\mathtt{transition\_vector}_{ij}$ 
being again its differential.
We also note that the coordinate representation at $\vec{\xi}$ of the Riemannian metric induced by the ambient embedding of the manifold can be easily expressed as the matrix $I_d + \nabla f_i(\vec{\xi})^\top \nabla f_i(\vec{\xi}).$
It remains to define a mechanism to determine chart membership. 
We accomplish this by learning the minimum-volume enclosing ellipsoid (MVEE) of $X_i$~\cite{nie_2005}, using a second-order cone program. 
More specifics regarding how the primitives in Algorithm~\ref{spec:atlas_graph} are implemented by \atlaslearn are given in Section~\ref{app:pointcloud}.


\section{Results}

\subsection{Online Subspace Learning}\label{sec:grassmann}
An \atlas data structure can be constructed explicitly for a manifold with a known atlas. We do this for the $(n,k)$-Grassmann manifold $\mathbf{Gr}_{n,k}$ with its canonical Ehresmann atlas. The implementation is a simple exercise in numerical algebra (Supp. Materials, Sec.~\ref{app:grassmann}). 
We use the resulting \atlas to compute the classic Grassmann inductive Fr\'echet expectation estimator (\texttt{GiFEE})~\cite{chakraborty_gifme}, a standard task in online subspace learning.
Our experiments show that our approach, based on quasi-Euclidean retractions on the \atlas data structure, outperforms state-of-the-art methods in runtime, with no loss in accuracy.

\paragraph{Fr\'echet Expectation Estimation}

Given a stream of samples $\mathcal{X}_1,\ldots,\mathcal{X}_i$ from a probability distribution on $\mathbf{Gr}_{n,k}$, the \texttt{GiFEE} estimator is computed by inductively applying the update rule
\begin{align}\label{eqn:log_update}
    \vec{v}_{i+1} & \gets \mathbf{Log}_{M_i}\left(\mathcal{X}_i\right)\\
    M_{i+1} & \gets \textbf{Exp}_{M_i}\left(\frac{1}{i+1}\vec{v}_{i+1}\right),
\end{align}
where $M_1=\mathcal{X}_1$ and the tangent vector $\vec{v}_i\in T_{M_i}\mathbf{Gr}_{n,k}$ is initialized to $\vec{v}_1=\vec{0}$. Chakraborty and Vemuri show that the \texttt{GiFEE} estimator converges in probability to the Fr\'echet mean of any distribution on $\mathbf{Gr}_{n,k}$ under certain limitations on support and Riemannian $L^2$-moment~\cite{chakraborty_gifme}. In this experiment, we use geodesic power distributions $\mathbf{GPD}(\mathcal{X}, p)$ with Fr\'echet mean $\mathcal{X}$ for $p > 1$ (formally defined in the Supplementary Materials).
To compute first-order updates on the \atlas representation, we replace invocations of $\mathbf{Exp}$ with invocations of quasi-Euclidean updates on the \atlas instance and invocations of $\mathbf{Log}$ with an \atlas-based approximation of the Riemannian logarithm, which are detailed in the Supplementary Materials (Sec.~\ref{app:grassmann}). 

\paragraph{Competitors} We compare our algorithm against these Grassmannian optimization routines: (1) the original GiFEE algorithm~\cite{chakraborty_gifme}; (2) \texttt{MANOPT}, the state-of-the-art package for manifold optimization~\cite{boumal_2016,em_alt,luchnikov_2021}; and 3)  \texttt{MANOPT-RET}, a method in the \texttt{MANOPT} package that uses a more efficient retraction.
In contrast, \texttt{GiFMEE} and \texttt{MANOPT} use closed forms for logarithms and exponential maps in $\mathbf{Gr}_{n,k}$,
but differ in how the numerical steps are organized, leading to different performance.
More details on the specific implementations are given in the Supplementary Materials.

\begin{figure}[h]
    \centering
    \includegraphics[width=\columnwidth, trim=0 0 0 30, clip]{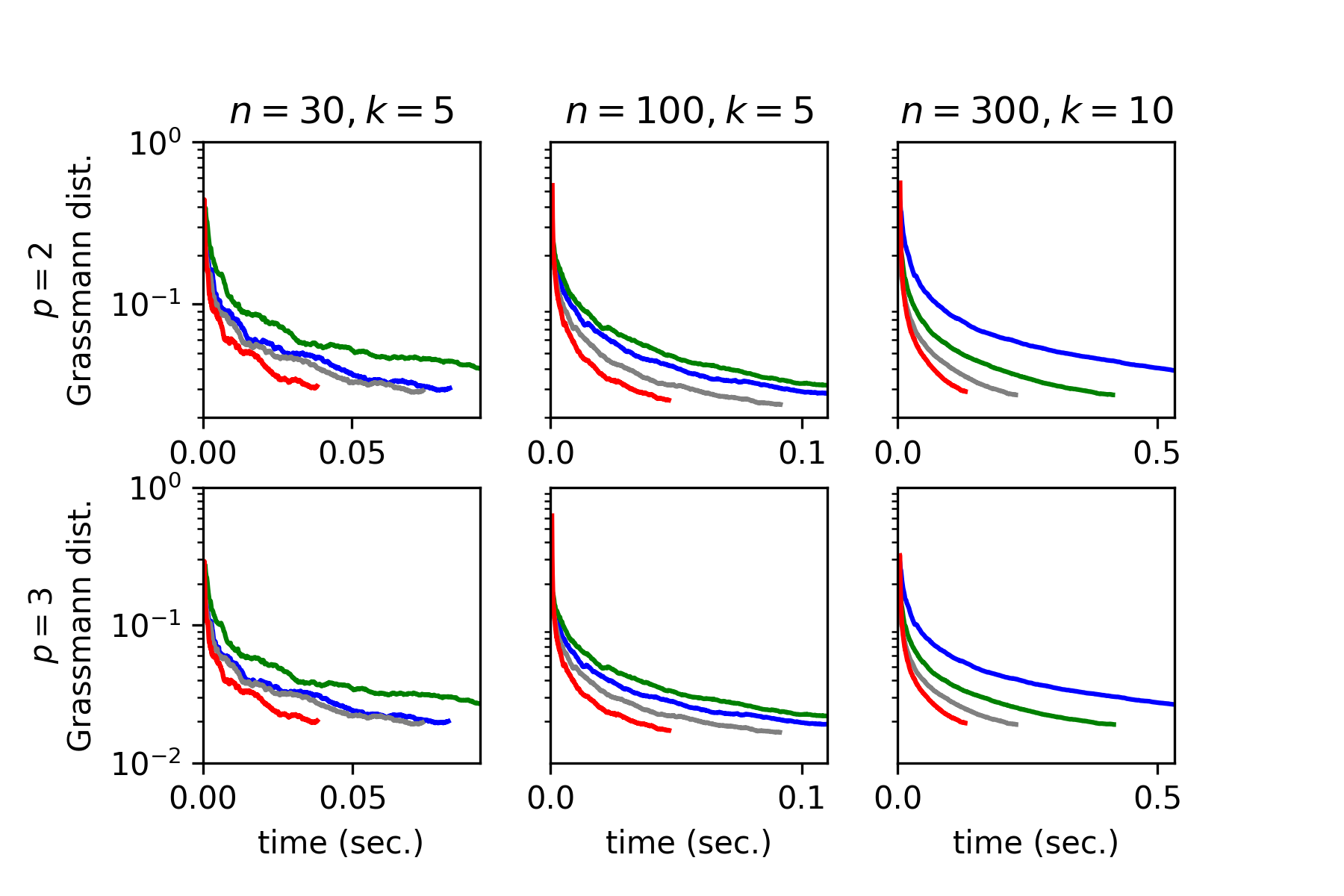}
    \includegraphics[width=\columnwidth, trim=60 30 60 30, clip]{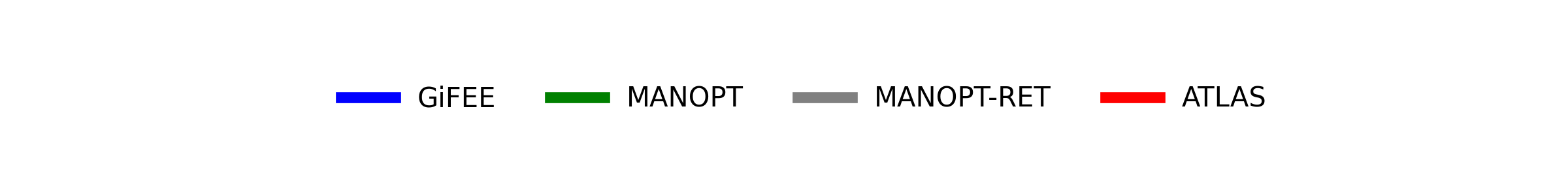}
    
    \caption{\textbf{Quasi-Euclidean updates in the \atlas representation of the Grassmannian $\mathbf{Gr}_{n,k}$ converge faster to the population Fr\'echet mean than other first-order update schemes.} Logarithms and retractions are executed by each method one observation at a time for 1,000 iterations, each observation being a $\mathbb{R}^{n, k}$ Stiefel matrix. Stiefel matrices $X \in \mathbb{R}^{n \times k}$ drawn from the distribution $\mathbf{GPD}(\mathcal{X},p)$ with fixed Fr\'echet mean $\mathcal{X}$ and $p\in\{2,3\}$ (Sec.~\ref{sec:geodesic_power_distribution}). $\mathbf{Gr}_{n,k}$-distance to $\mathcal{X}$ from each method's iterate is plotted against cumulative runtime.}
    \label{fig:table_high_nk}
\end{figure}

\paragraph{Experimental Results} All methods have similar, high accuracy, measured by geodesic distance between the \texttt{GiFEE} estimator and the true population Fr\'echet mean as a function of the number of iterations. The experimental results show runtime superiority of our 
\atlas approach over the other first-order routines for experiments with $(n,k)$ set to $(30,5)$, $(100,5)$, or $(300,10)$ (Fig.~\ref{fig:table_high_nk}). 


\subsection{Preservation of Manifold Geometry}\label{sec:comparison}

Here, we evaluate how well the \atlas data structure constructed by \atlaslearn preserves topological and geometric features of the ground-truth manifold for realistic data. Because it has a well-studied structure, we leverage the space of $3\times3$ high-contrast, natural image patches \cite{vanhaterenIndependentComponentFilters1998}, the underlying manifold of which Carlsson elegantly parametrized and showed was homeomorphic to the Klein bottle~\cite{carlsson}. Focusing on the $k_0$ parametrization~\cite{sritharan}, which induces a Riemannian metric through a specific embedding into $\mathbb{R}^{3 \times 3}$, we 
construct meshes of different cardinalities on the manifold  and use them as inputs for \atlaslearn and the dimensionality-reduction methods \texttt{UMAP}, \texttt{TSNE}, \texttt{LLE}, \texttt{Isomap} and \texttt{PCA}.
All methods are run with target dimension $D=2$, the intrinsic dimension of the manifold. We then compute persistent homology and geodesic distances on the resulting representations.

\begin{figure}[h]
    \centering
    \includegraphics[width=0.9\columnwidth, trim=0 20 0 20, clip]{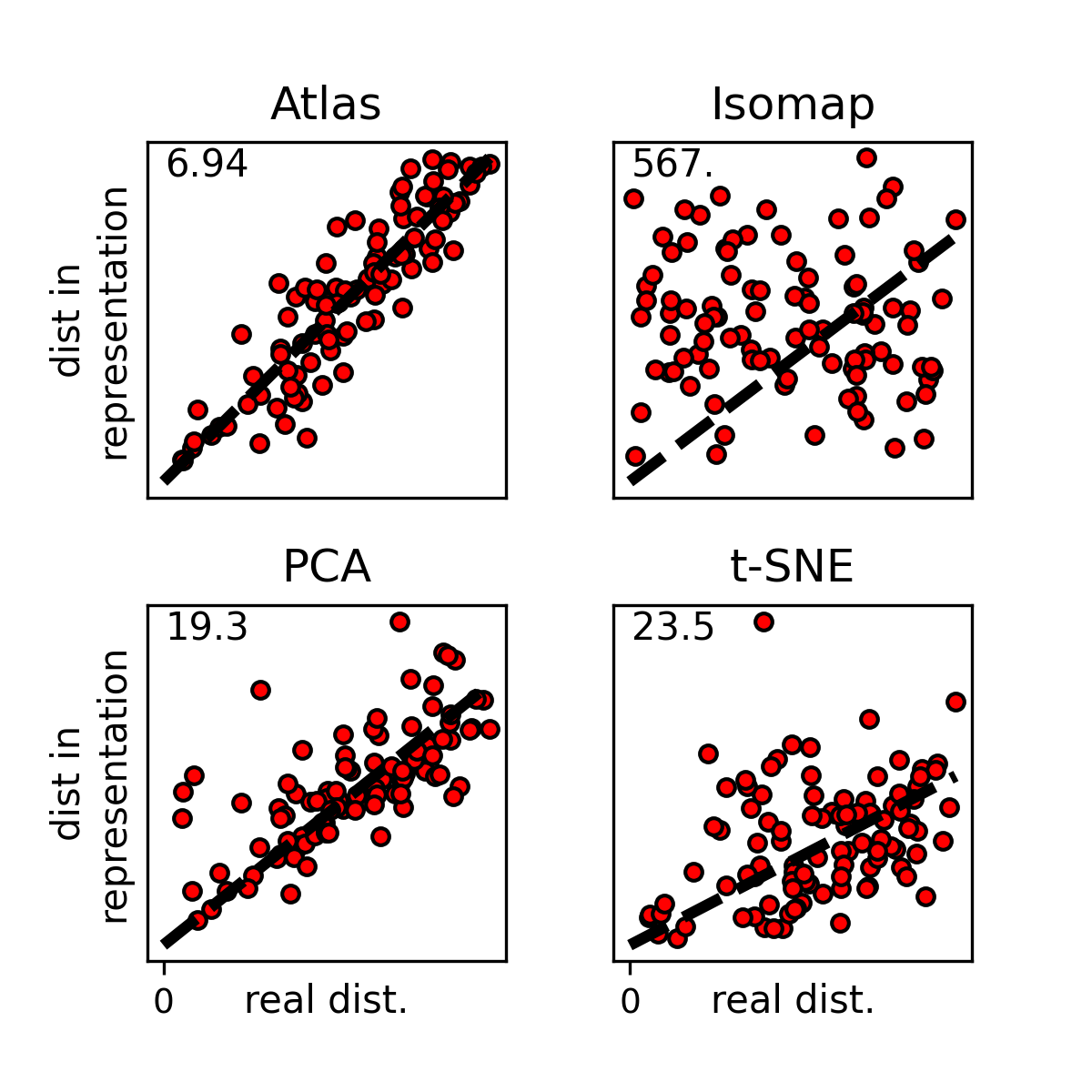}
    \caption{\textbf{The \atlas data structure preserves geodesic distances  better than common dimensionality-reduction techniques.} 
    \texttt{Isomap}, \texttt{PCA} and \texttt{$t$-SNE} were computed on a mesh of 1 million points. \atlaslearn was asked to produce 64 charts, given only a 10,000-points mesh. For 100 randomly sampled pairs of points on the Klein bottle, each scatter plot shows each pair (dot) according to a precise estimate of their true geodesic distance (x axis, ``real dist.'') versus the embedding distance (y axis, ``dist. in representation'') for each embedding (panel title) (Supp. Materials, Sec.~\ref{app:klein}). Metric distortion given in top left corner. Dashed line, fit with $y$-intercept zero.}
    \label{fig:dist_fig}
\end{figure}

The results show that geodesic distances between pairs of points are better preserved by the \atlas representation than by the other methods (Fig.~\ref{fig:dist_fig}).
In particular, \texttt{Isomap}, which explicitly aims to preserve geodesic distances, fares very poorly.
These findings, robust even larger dimension budget $D>2$ is allowed (Supp. Materials, Sec.~\ref{app:klein}), likely stem from the nontrivial topology and curvature of the Klein bottle.
Indeed, we find that while competitors do not preserve homological features, as measured by an aggregate bottleneck distance~\cite{oudot}, the \atlas representation are almost perfectly recapitulates these features (Supp. Materials, Sec.~\ref{app:klein}).  
%

\subsection{Classification of Image Patches}\label{sec:classification_patches}

\begin{figure}[h]
    \centering
    \includegraphics[width=\columnwidth, trim=90 30 50 50, clip]{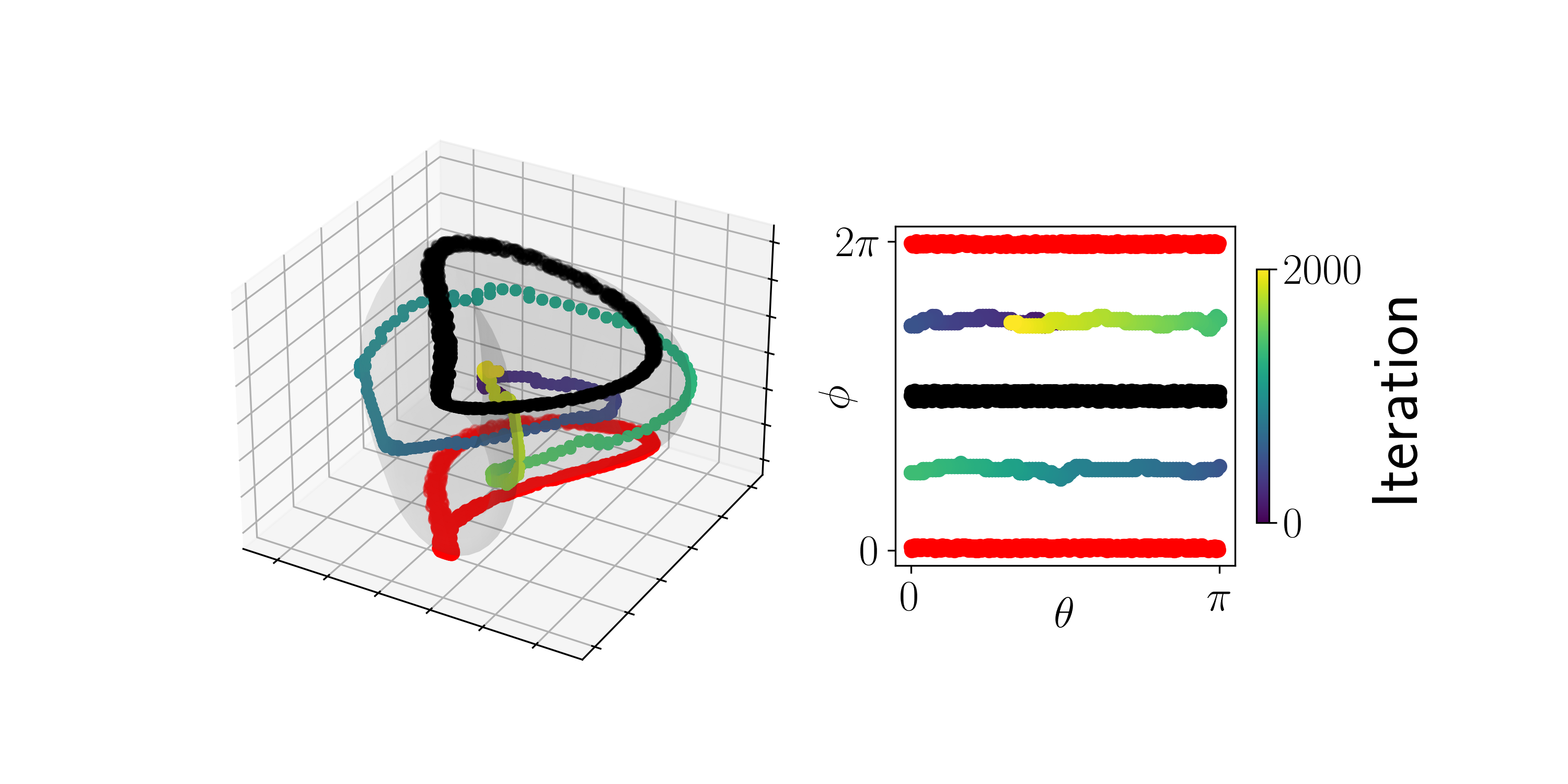}
    
    \caption{\textbf{The \texttt{RPB} algorithm effectively learns a discriminating boundary on the manifold of high-contrast image patches.} A 64-chart \atlas data structure was generated by \texttt{Atlas-Klein} from a 10,000-point mesh of the Klein bottle manifold for high-contrast image patches. The \texttt{RPB} implementation was then applied to learn a discriminating boundary (colored by discretized ODE iterations of the algorithm) between the convex (red) and concave (black) patches, as illustrated in Karcher (left) and polar (right) coordinates.}
    \label{fig:rpb}
\end{figure}

To test the ability of a learned \atlas to enable Riemannian optimization and machine learning, we implemented and applied the Riemannian principal boundary (\texttt{RPB}) algorithm~\cite{yao_2020}, which generalizes support vector machines (\texttt{SVM}) to Riemannian geometry.
The natural image patches have a natural principal boundary between convex and concave patches, which we aim to identify using \texttt{RPB} algorithmic approach (Fig.~\ref{fig:rpb}, Supp. Materials, Sec.~\ref{app:klein}).
For a 2-dimensional manifold $\mathcal{M}$, the \texttt{RPB} algorithm learns a binary classifier by first learning one-dimensional ``boundary'' submanifolds $\Gamma,\Gamma^\prime\subset\mathcal{M}$, for each of the two classes, and then ``interpolating'' between these two boundaries to create a 1-dimensional separating submanifold. The interpolation is done by characterizing $\Gamma$ and $\Gamma^\prime$ as solutions to  ordinary differential equations (ODEs) called \emph{principal flows} parameterized by the class label. At each iteration of the algorithm, the solution curve to each principle flow is approximated locally as a short geodesic update in the direction of the first derivative given by the ODE. The boundary curve is simultaneously updated by parallel-transporting the first derivatives from the two principle flows, taking their weighted Fr\'echet mean, and taking a small geodesic update in the direction of this average. The resulting boundary curve, which is referred to as the \textit{principal boundary}, serves as a binary classifier. 

Our \texttt{RPB} algorithm implementation applies Euler's method over 2,000 iterations to simultaneously integrate the principal flows for the convex and concave patches, as well as the principal boundary ODE.  
For easier visualization in polar coordinates, the \atlas is built using \texttt{Atlas-Klein}, a deterministic variant of \atlaslearn with 64 chart centers and spherical bounding ellipsoids that are specified manually using a Cartesian mesh in $k_0$-polar coordinates (Supp. Materials, Sec.~\ref{app:klein}). 

Our analyses show that \texttt{RPB} successfully learns a boundary curve between the convex and concave patches, solely in terms of the geometry encoded in the \atlas representation (Fig.~\ref{fig:rpb}). The result is a robust separator that is interpretable, i.e., entirely composed of high-contrast patches, which other techniques fail to produce (detailed in Supp. Materials, Sec~\ref{app:rpb}).

\subsection{Integrating Vector Flows Over Transcriptomic Space}\label{sec:bio}

Finding that the \atlas performs well for manifolds with atlases that are either known or learned from relatively low-dimensional point-cloud data, we turn to a more complex setting. In this section, we explore the utility of a learned \atlas representation in the study of dynamics over real-world single-cell transcriptomic data. Specifically, we consider an application in which vector field information, representing estimated time derivatives of the transcriptional states of cells, known as ``RNA velocities''~\cite{frank_rec, qiuMappingTranscriptomicVector2022, wang_velocity_2025, topicvelo}, have been inferred and we wish to simulate the resulting RNA dynamics.

A benchmark dataset for this task from Qiu \textit{et al.} contains transcriptome-wide mRNA expression data from CD34+ hematopoietic stem cells (HSCs) and progenitor cells~\cite{qiuMappingTranscriptomicVector2022}.
%
Their computational framework \texttt{Dynamo} leverages metabolic labeling, which distinguishes nascent from pre-existing mRNA, to estimate more accurate RNA velocity vectors for each datum (cell). They then use sparse kernel regression to extend these velocities to a vector field  $\vec{V}_{\texttt{amb}}$ over the ambient space $\mathbb{R}^{30}$, determined by the top 30 principal components. 
Finally, they integrate this vector field as an ODE in ambient space to reveal transcriptional trajectories characterizing the differentiation of HSCs into terminal cell types.

An explicit assumption behind this approach is that integration over the vector field, starting at HSC data points, should maintain proximity to the latent lower-dimensional manifold $\mathcal{M}$ of mRNA expression.
To verify this assumption, we compared the ODE solution yielded by \texttt{Dynamo}'s ambient-space integration against that yielded by integrating the ODE over an \atlas learned via \atlaslearn.
To perform this integration, we endow our \atlas with a vector field $\vec{V}_{\texttt{atlas}}$ given by the projection of $\vec{V}_{\texttt{amb}}$ onto the tangent spaces produced by \atlaslearn .
The complete setup is detailed in the Supplementary Materials (Sec.~\ref{app:rna}).

The results show that the iterates of the \atlas-restricted ODE lie significantly closer to datapoints than ambient ODE iterates do (Fig.~\ref{fig:dynamo_iterates}), suggesting that the ambient ODE veers away from the manifold.
Moreover, the average distance between an \atlas-restricted iterate and the closest datum  is relatively small, falling below the global average distance between data  points and the learned \atlas manifold. 
%
%
These findings suggest that the learned \atlas instance serves as a constraint that enables the integration along $\vec{V}_{\texttt{amb}}$  to reconstruct transcriptional dynamics in HSC differentiation more faithfully.

\begin{figure}
    \centering
    \includegraphics[width=\columnwidth, trim=90 0 28 0, clip]{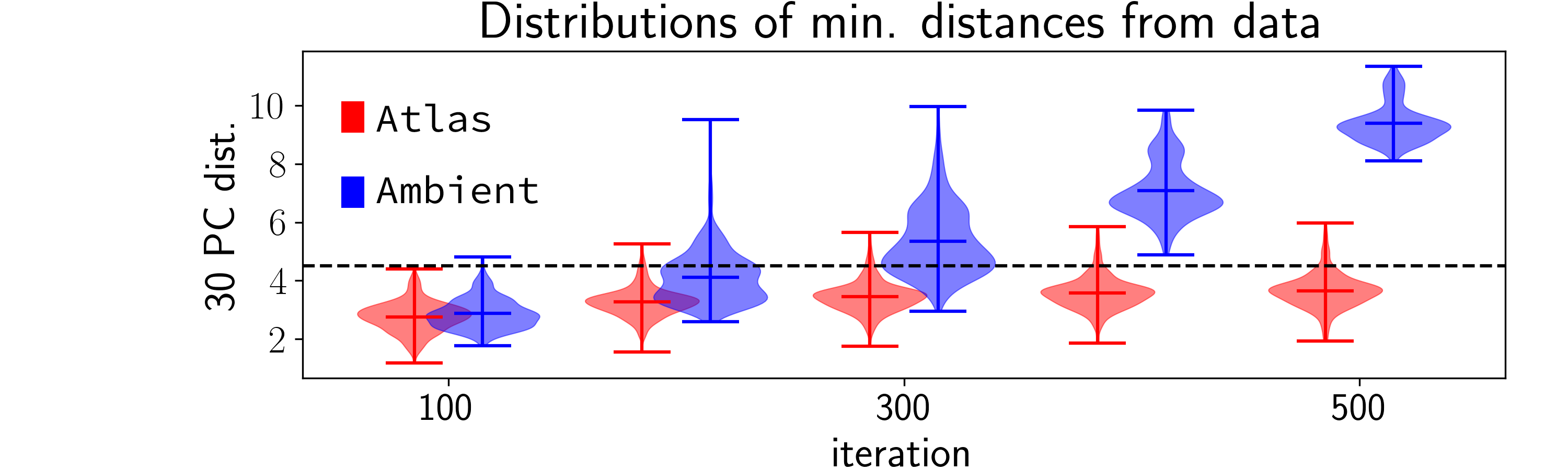}
     \caption{\textbf{Integrating the RNA velocity vector field over the \atlas representation rather than the ambient space more closely matches observations in hematopoietic data.}
     For 309 distinct initial conditions, corresponding to labeled HSCs, a vector field over human hematopoietic cells is integrated, either over the ambient space (blue) or over a 5-dimensional \atlas representation with 30 charts (red).  At every iteration, distances from each of the 309 evolutions to the closest datum in 30-PC space is computed; these are shown as violin plots. Dashed line indicates average distance between a datum and nearest \atlas point.}
    \label{fig:dynamo_iterates}
\end{figure}

\section{Discussion}

We present an atlas-based approach to manifold learning and Riemannian optimization that is built upon a minimalist representation of a differentiable atlas. Surprisingly, our methods are able to preserve nontrivial homology and metric information in a way that state-of-the-art dimensionality reduction techniques cannot achieve. 
They also enable the implementation of Riemannian optimization routines on manifolds both known and learned from point-cloud data.

One weakness of our current definition of \atlas is the lack of support for varifolds, generalizations of manifolds whose intrinsic dimensionality may change from location to location. This is important for real-world data that are unlikely to conform to a constant dimension. Another opportunity for improvement lies in the construction of charts within \atlaslearn, which could benefit from adaptively choosing chart centers by considering the goodness of fit of the
resulting quadratic approximations. Finally, robust Riemannian optimization on an atlas-based representation requires careful consideration
of the error introduced through both geodesic approximation and transition maps, which remains to be evaluated rigorously. 


\section*{Acknowledgments}
SJR and LO gratefully acknowledge the support of the NSF-Simons National Institute for Theory and Mathematics in Biology via grants NSF DMS-2235451 and Simons Foundation MP-TMPS-00005320. SJR is a CZ Biohub Investigator. RAR was supported in part by a National Science Foundation Graduate Research Fellowship. The authors wish to thank Shmuel Weinberger for his feedback and for many inspiring conversations.

\bibliographystyle{plain} 
\bibliography{references}


\clearpage
\appendix
\thispagestyle{empty}
\onecolumn
\aistatstitle{Atlas-based Manifold Representations\\ for Interpretable Riemannian Machine Learning: \\
Supplementary Materials}

\begin{bibunit}[plain]
\section*{Clarifications}

In the body of this paper, we list $\texttt{coord\_chart}_i$ as a method of the \atlas instance that implements the computation of $\varphi_i$. We use ``$\texttt{coord\_chart}_i$'' and ``$\varphi_i$'' interchangeably in the body, though they refer to the same thing abstractly. In the supplement, and in future versions of this paper, we will limit ourselves to ``$\varphi_i$'' for simplicity.

Further, in the body, we use $k_0$ to refer to the map $k^0$ used to parametrized the Klein bottle in~\cite{sritharan}. For the sake of consistency with the body, we continue to use $k_0$ in the supplement.

\section*{Errata}
In Section~\ref{subsec:atlaslearn}, we use $\Psi_{ij}$ to denote the transition map from $\mathcal{V}_i$ to $\mathcal{V}_j$ and define
$$
\Psi_{ij}(\vec{\xi}) =  L_j^\top \left[ \left(L_i \vec{\xi} + \vec{m}_{i}\right) - \vec{m}_{ j}\right].
$$
While this is a plausible choice, it is actually incorrect as it does not make use of the quadratic approximation to the manifold. The correct formula is:
$$
\Psi_{ij}(\vec{\xi}) =  L_j^\top \left[\frac{1}{2}M_i K_i\left(\vec{\xi}\otimes \vec{\xi}\right) + \left(L_i \vec{\xi} + \vec{m}_{i}\right) - \vec{m}_{ j}\right].
$$
In the body, we use the notation $\Psi_{ij}$ and $\psi_{ij}$ to refer to the transition map; in the supplement, we restrict ourselves to $\psi_{ij}$ to reduce confusion.

\section*{Code Repository}

All code is available at the anonymized repository \url{https://anonymous.4open.science/r/atlas\_graph\_learning-6DE0}.

\section{Details and Mathematical Results on \atlas}
\label{app:atlas}

\subsection{\atlas discrepancy}

Recall the following definition:
\begin{definition}
Given an \atlas data structure $\mathcal{A}$ with charts $\varphi_i: \mathcal{V}_i \to \mathbb{R}^D$ and transition maps $\psi_{ij}: \mathcal{V}_{ij} \to V_j,$ the discrepancy of  $\mathcal{A}$ is defined as 
$$
\sup_{i,j} \sup_{\vec{\xi} \in \mathcal{V}_{ij}} \|\varphi_i(\vec{\xi}) - \phi_j(\psi_{ij}(\vec{\xi}))\|_{\mathbb{R}^D}.
$$
\end{definition}

As claimed in the body, if the discrepancy of an \atlas $\mathcal{A}$ is $0$, then $\mathcal{A}$ is a differentiable atlas for a manifold $\cM \subset \mathbb{R}^D.$

\begin{lemma}
If the discrepancy of \atlas $\mathcal{A}$  equals $0$, then for all charts $i,j$ 
$$
\psi_{ij} = \varphi_j^{-1} \circ \varphi_i
$$
on $\mathcal{V}_{ij}.$
\end{lemma}
\begin{proof}
 By the definition of discrepancy, for all charts $i,j$, for all $\vec{\xi} \in \mathcal{V}_{ij}$, we have
 $$
\|\varphi_i(\vec{\xi}) - \phi_j(\psi_{ij}(\vec{\xi}))\|_{\mathbb{R}^D} = 0 \implies \varphi_i(\vec{\xi}) = \phi_j(\psi_{ij}(\vec{\xi})) \implies \psi_{ij}(\vec{\xi}) = (\varphi_j^{-1} \circ \varphi_i)(\vec{\xi}).
 $$
\end{proof}

In Section~\ref{sec:app_atlas_learn_convergence},  we sketch the proof that, given $n$ samples from the uniform distribution over a smooth manifold $\cM$ embedded in $\mathbb{R}^D,$ \atlaslearn builds an \atlas data structure $\mathcal{A}$ whose discrepancy goes to $0$ as $n$ goes to infinity, showing that $\mathcal{A}$ converges to a differentiable atlas.

\subsection{Representation of the Riemannian metric on \atlas}\label{sec:atlas_graph_met}
Let $\left(\mathcal{V},\varphi\right)$ be a coordinate chart of Riemannian manifold $\mathcal{M}$. If $\mathfrak{g}$ is the Riemannian metric on $\mathcal{M}$, then $\mathfrak{g}_p$ can be thought of as a linear isomorphism from $T_p\mathcal{M}$ to $T_p^*\mathcal{M}$ (i.e., the vector space of linear functionals on $T_p\mathcal{M}$). For $\vec{\xi}=\varphi(p)$, we can define a distinct Riemannian metric $\overline{\mathfrak{g}}_{\vec{\xi}}:T_{\varphi(p)}\mathcal{V}\to T_{\varphi(p)}^*\mathcal{V}$ in terms of local coordinate charts such that
\begin{equation}\label{eqn:g_bar}
\overline{\mathfrak{g}}_{\vec{\xi}}(\vec{\tau}):\vec{\tau}^\prime\mapsto\mathfrak{g}_p\left(\left[D\varphi^{-1}\right]_{\varphi(p)}\left(\vec{\tau}\right)\right)\left(\left[D\varphi^{-1}\right]_{\varphi(p)}\left(\vec{\tau}^\prime\right)\right),
\end{equation}
where $D\varphi^{-1}$ denotes the differential of $\varphi^{-1}$.

Similarly, for an \atlas data structure, we approximate the metric $\overline{\mathfrak{g}}_{\vec{\xi}}(\vec{\tau})$ for $\vec{\xi} \in \mathcal{V}_i$ as:
$$
\overline{\mathfrak{g}}_{\vec{\xi}}(\vec{\tau}) : \vec{\tau}' \mapsto \tau^\top(\mathtt{D\_chart}(\vec{\xi})^{-1})^\top (\mathtt{D\_chart}(\vec{\xi})^{-1}) \tau'
$$
where we took $\mathfrak{g}_p$ to be the identity, i.e., the standard metric for the Euclidean space $\mathbb{R}^D$ in which the manifold $\cM$ is embedded.

\section{Approximation of differential-geometric primitives in approximately geodetic coordinates}\label{app:approx_primitives}

\subsection{Approximation error of vector transport on \atlas}\label{app:vec_trans}

When a vector field $V$ acts on a differentiable manifold $\mathcal{M}$ to take a point $p$ to $q$, vectors in $T_p\mathcal{M}$ should be transformed into vectors in $T_q\mathcal{M}$. There are many linear isomorphisms between $T_p\mathcal{M}$ and $T_q\mathcal{M}$, and the ``correct'' choice of isomorphism should be uniquely determined by the path $\gamma:[0,1]\to\mathcal{M}$ along which $V$ takes $p$ to $q$. This isomorphism $T_p\mathcal{M}\to T_q\mathcal{M}$ determined by $\gamma$ is called the \textit{parallel transport}. Its fundamental role in first-order methods over manifolds is described in~\cite{hosseini_and_sra}.

For \atlas, within a chart, we approximate parallel transport $\mathcal{P}^{\mathbf{Ret}_p}_{p\to\mathbf{Ret}_p\left(\vec{v}\right)}\vec{w}$ of tangent vector $\vec{w}\in T_p\mathcal{M}$ to $T_{\mathbf{Ret}_p\left(\vec{v}\right)}\mathcal{M}$ along $\mathbf{Ret}_p$ with the identity vector transport 
\begin{equation}\label{eq:vec_transport}
\mathcal{T}_{\vec{\xi}\to\vec{\xi}+\vec{\tau}}:\vec{\sigma}\mapsto\vec{\sigma}
\end{equation}
It is used to approximate Riemannian logarithms (Secs.~\ref{app:grassmann} and~\ref{app:rpb}), to approximate geodesic path lengths (Sec.~\ref{app:klein}) and implement the Riemannian principal boundary (RPB) algorithm (Sec.~\ref{app:rpb}).

Here, we look at the error term of the identity vector transport of Equation~\ref{eq:vec_transport} as an approximation of the parallel transport of a tangent vector along a quasi-Euclidean update. 

\begin{theorem}
    The vector transport of Equation~\ref{eq:vec_transport} approximates parallel transport along quasi-Euclidean updates up to order $O\left(\left\lVert\vec{\sigma}\right\rVert_{\overline{\mathfrak{g}}}\left\lVert\vec{\tau}\right\rVert_{\overline{\mathfrak{g}}}^2\right)$.
\end{theorem}
\begin{proof}
Let $\left(\mathcal{U},\mathcal{V},\varphi\right)$ be a coordinate chart of Riemannian manifold $\mathcal{M}$ with points $p,q\in\mathcal{U}$ satisfying $p\neq q$. Further, let $\tau:=\varphi(q)-\varphi(p)$, and let there be a tangent vectors $\vec{w}\in T_p\mathcal{M}$ with representative tangent vector $\vec{\sigma}_0:=\left[D\varphi\right]_p\left(\vec{w}\right)$. Lastly, let $\Gamma^\lambda_{\mu\nu}$ be the Christoffel symbol of the second kind. The parallel transport of vector field $\sigma$ along path the quasi-Euclidean update from $\varphi(p)$ to $\varphi(q)$ is given \cite{absil} by
\begin{equation*}
\tau^j\partial_j\left(\sigma^l\right)\partial_l+\tau^j\sigma^k\Gamma^l_{jk}\partial_l=0.
\end{equation*}
Recognizing that the time derivative of the $j$th component of $\sigma$, which we denote by $\dot{\sigma}^j$, is equal to $\tau^l\partial_l\left(\sigma^j\right)$, we see that the $j$th component of $\mathcal{P}_{p\to q}^{\mathbf{Ret}_p}\vec{w}$ is given by
\begin{equation}\label{eqn:component_of_par}
    \left[\mathcal{P}_{p\to q}^{\mathbf{Ret}_p}\vec{w}\right]^j=\left(\left[D\varphi^{-1}\right]_{\varphi(q)}\right)^j_k\sigma_0^k-\left(\left[D\varphi^{-1}\right]_{\varphi(q)}\right)^j_k\int_0^1\sigma^l\Gamma^k_{ml}\tau^mdt.
\end{equation}
Therefore, the following series of deductions ends with a convenient form for the $j$th component of the difference between $\mathcal{P}_{p\to q}^{\mathbf{Ret}_p}\vec{w}$ and $\mathcal{T}_{p\to q}\vec{w}$.
\begin{align*}
    \left[\mathcal{T}_{p\to q}\vec{w}-\mathcal{P}_{p\to q}^{\mathbf{Ret}_p}\vec{w}\right]^j&=\left(\left[D\varphi^{-1}\right]_{\varphi(q)}\right)^j_k\int_0^1\sigma^l\Gamma^k_{ml}\tau^mdt \\
    &=\left(\left[D\varphi^{-1}\right]_{\varphi(q)}\right)^j_k\tau^m\int_0^1\sigma^l\Gamma^k_{ml}dt \\
    &=\left(\left[D\varphi^{-1}\right]_{\varphi(q)}\right)^j_k\left(\tau^m\sigma_0^l\int_0^1\Gamma^k_{ml}dt+\tau^m\int_0^1\left(\sigma^l-\sigma_0^l\right)\Gamma^k_{ml}dt\right)
\end{align*}
The $\left[D\varphi^{-1}\right]$ term is absorbed according to Equation~\ref{eqn:g_bar} when computing the Riemannian metric, and the term $\tau^m\sigma_0^l\int_0^1\Gamma^k_{ml}dt$ is of order  $O\left(\left\lVert\vec{\sigma}_0\right\rVert_{\overline{\mathfrak{g}}}\left\lVert\vec{\tau}\right\rVert_{\overline{\mathfrak{g}}}\right)$. To see that $\mathcal{T}_{p\to q}\vec{w}$ approximates $\mathcal{P}_{p\to q}^{\mathbf{Ret}_p}\vec{w}$ up to order $O\left(\left\lVert\vec{\sigma}_0\right\rVert_{\overline{\mathfrak{g}}}\left\lVert\vec{\tau}\right\rVert^2_{\overline{\mathfrak{g}}}\right)$, it remains to show that this is the order of the term $\tau^m\int_0^1\left(\sigma^l-\sigma_0^l\right)\Gamma^k_{ml}dt$. This is accomplished in the following series of deductions, which make use of Lemma~\ref{lem:sig_tau_induction} and the Taylor expansion of the integrand.

\begin{align*}
    \tau^m\int_0^1\left(\sigma^l-\sigma_0^l\right)\Gamma^k_{ml}dt&=\tau^m\int_0^1\left(t\dot{\sigma}^l_0+\sum_{d=2}^\infty\frac{t^d}{d!}\left(\sigma^{(d)}_0\right)^l\right)\Gamma^k_{ml}dt \\
    &=\int_0^1\left(tO\left(\left\lVert\vec{\sigma}_0\right\rVert_{\overline{\mathfrak{g}}}\left\lVert\vec{\tau}\right\rVert^2_{\overline{\mathfrak{g}}}\right)+\sum_{d=2}^\infty \frac{t^d}{d!}O\left(\left\lVert\vec{\sigma}_0\right\rVert_{\overline{\mathfrak{g}}}\left\lVert\vec{\tau}\right\rVert^{d+1}_{\overline{\mathfrak{g}}}\right)\right)dt \\
    &=O\left(\left\lVert\vec{\sigma}_0\right\rVert_{\overline{\mathfrak{g}}}\left\lVert\vec{\tau}\right\rVert^2_{\overline{\mathfrak{g}}}\right)
\end{align*}
\end{proof}

\begin{lemma}\label{lem:sig_tau_induction}
    Let $\vec{\sigma}_0^{(d)}$ be the $d$th derivative of $\vec{\sigma}$ with respect to time, evaluated at time $t=0$. For all $d\in\mathbb{N}_{>0}$, it holds that $\vec{\sigma}_0^{(d)}$ is linear in $\sigma_0$ and homogeneously of order $d$ in both $\vec{\tau}$ and the Christoffel symbols of the second kind.
    \begin{proof}
        This is easy to prove by induction, with the base case $d=1$ being given by the parallel transport equation.
    \end{proof}
\end{lemma}

\subsubsection{Approximate geodesic distances via quasi-Euclidean updates}\label{sec:app_naive_dist}
In Riemannian optimization algorithms that iterate over manifold-valued data, a loss function may weigh the contribution of each datum according to its distance from the current iterate, e.g., stochastic Riemannian gradient descent~\cite{hosseini_and_sra}. In regimes where quasi-Euclidean updates approximate the exponential map well, the paths traversed by quasi-Euclidean updates approximate geodesic paths well. Therefore, the lengths of quasi-Euclidean paths with respect to the metric inherited from the ambient space can be used to estimate (albeit overestimate\footnote{Overestimation results from the definition of geodesics as length-minimizing paths on a Riemannian manifold}) geodesic distances. For $\tv_0,\tv_1\in\mathcal{V}_i$, we refer to
\begin{equation}\label{eqn:naive_approx_dist}
    \tilde{d}_i\left(\tv_0,\tv_1\right)=\int_0^1\left(\left(\vec{\xi}_1-\vec{\xi}_0\right)^\top\left(\vec{\xi}_1-\vec{\xi}_0\right)+\sum_{j=1}^{D-d}\left(\left(\vec{\xi}_1-\vec{\xi}_0\right)^\top K_{ij}\left((1-t)\tv_0+t\tv_1\right)\right)^2\right)^{\frac{1}{2}}dt
\end{equation}
as the \textit{na\"ive approximate distance} between $\tv_0$ and $\tv_1$.

To define the naive approximate distances between two points $\vec{\xi}_1\in\mathcal{V}_1,\vec{\xi}_2\in\mathcal{V}_2$ that are not in the same coordinate chart we use the shortest-path distance between their ambient representations $\vec{x}_1$ and $\vec{x}_2$, respectively.
These will be an accurate representation of the corresponding manifold distance only when $\vec{x}_1$ and $\vec{x}_2$ are close .

We construct a graph $G_{\delta, \epsilon}$ whose vertices are the union of $\delta$-nets of each chart domain $\mathcal{V}_i$. We connect vertices by an edge when their na\"ive approximate distances are less than $\epsilon$ and assign the distance to be the edge length.
Finally, we compute approximate geodesics between two points $\vec{\xi}_1\in\mathcal{V}_1$ and $\vec{\xi}_2\in\mathcal{V}_2$ by rounding each point to the nearest representative in that chart's $\delta$-net and performing a shortest-path computation in $G_{\delta, \epsilon}$ between the resulting vertices.

\section{Methods for inferring coordinate chart structure from point clouds}\label{app:pointcloud}

\subsection{Time complexity of \atlaslearn}

The \atlaslearn algorithm is described in Figure~\ref{alg:learn_atlas_graph}.

\begin{algorithm}[h]
    \caption{\atlaslearn (Learn \atlas from point cloud data)} \label{alg:learn_atlas_graph}
    \begin{algorithmic}
        \Require Point cloud $X\in\mathbb{R}^{N \times D}$ in $D$ dimensions
        \Require Estimated intrinsic dimension $d$
        \Require Number $k$ of charts to learn; number $\nu$ neighbors in nearest neighbors graph

        \State $G \gets \texttt{NearestNeighbors}(X, \nu)$ \Comment{$l_2$-weighted, $\nu$-nearest neighbors graph}
        
        \State $X_1, \ldots, X_k \gets \texttt{KMedoids}(G, k)$ \Comment{Partition $X$ by $k$-medoids}
        
        \State $\mathcal{A} \gets \varnothing$ \Comment{Store charts}
        
        \For{$j \in \{1, \ldots, k\}$}
            \State $\vec{x}, L, M, K \gets \texttt{LocalQuadraticApproximation}\left(X_j\right)$ \Comment{Follows methodology described in Sec.~\ref{app:quad_approx_pt_cloud}}
            \State $A, \vec{b}, c \gets \texttt{MVEE}\left(X_j\right)$ \Comment minimum-volume ellipsoid enclosing $X_j$; takes form $\vec{x}^\top A \vec{x} +\vec{b}^\top \vec{x} + c = 0$
            \State $\mathcal{A} \gets \mathcal{A} \cup \left\{\left(\vec{x}, L, M, K, A, \vec{b}, c\right)\right\}$ \Comment{$\vec{x}, L, M, K$ encode local Riem. structure, $A, \vec{b}, c$ transition boundary}
        \EndFor
        
        \State \Return $\mathcal{A}$
    \end{algorithmic}
\end{algorithm}

Given a point cloud $X \in \mathbb{R}^{N \times D}$ and $\nu \in \mathbb{N}_{>0}$, the time complexity of constructing an $l_2$-weighted, $\nu$-nearest neighbors graph $G$ from $X$ is $O(Dn^2)$, though this can be sped up through approximate constructions (e.g.,~\cite{zhao_2022}). Running $k$-medoids using the shortest path distance on $G$ takes $O(kN^2)$ time using FasterPAM; variants of $k$-medoids yielding approximate solutions can reduce this to $O(kN)$~\cite{schubertFastEagerKmedoids2021}. The method \texttt{LocalQuadraticApproximation}, outlined in Section~\ref{app:quad_approx_pt_cloud}, takes the following time complexities in different subroutines:
\begin{description}
    \item[Learning Stiefel matrices $L$ and $M$ by singular value decomposition (SVD):] $O(N D d)$~\cite{talwalkar_2013}
    \item[Learning mean $\vec{x}$:] $O(N D)$
    \item[Learning trilinear form $K$:] $O(N d^2)$ time to construct $\mathbf{t}$, $O(N (D - d))$ time to construct $\mathbf{n}$, $O(N d^4)$ to construct the Moore-Penrose inverse $\mathbf{t}^\dag$ of $\mathbf{t}$ by SVD~\cite{li_2019}, and $O(N (D - d) d^2)$ time to multiply $\mathbf{t}^\dag$ and $\mathbf{n}$; therefore, $O(N d^4 + N (D - d) d^2)$ time overall.
\end{description}

\subsection{Quadratic approximations of point clouds}
~\label{app:quad_approx_pt_cloud}

For local quadratic approximation of point clouds, there exists a procedure (Section  C.1,\cite{sritharan}), which we include here for completeness, with minor modifications in notation.
Let $X\in\mathbb{R}^{N\times D}$ be a matrix whose rows are observations in $\mathbb{R}^D$. In practice, we assume that $X$ only contains points sufficiently close to one another for a local quadratic polynomial approximation to be reasonable. 
Lastly, let $\tilde{X}\in\mathbb{R}^{N\times D}$ be the matrix such that, if $\vec{x}_i$ is the $i$th row of $X$, then the $i$th row of $\tilde{X}$ is equal to $\vec{x}_i - \overline{\vec{x}}$, where $\overline{\vec{x}}$ is the mean of $\vec{x}_i$.
We can construct the local covariance matrix
\begin{equation}\label{eqn:local_cov}
    \Sigma = \frac{1}{N} \tilde{X}^\top \tilde{X}.
\end{equation}
We can get an eigendecomposition $V\Lambda V^\top=\Sigma$ such that the diagonal entries of $\Lambda$ are decreasing. Fixing $d<D$, we can define $L \in \mathbb{R}^{D\times d}, M \in \mathbb{R}^{D\times (D-d)}$ satisfying $V = \left(L \mid M \right)$.

\textbf{The decomposition $V=\left(L \mid M \right)$ allows us to decompose $\mathbb{R}^D$ into \textit{tangential coordinates} with basis given by the columns of $L$ and \textit{normal coordinates} with basis given by the columns of $M$.} The mean-centered data $\tilde{X}$ similarly decompose into a \textit{tangential component} $\mathbf{\tau}=\tilde{X}L$ and a \textit{normal component} $\mathbf{n}=\tilde{X} M_{\vec{c},r}$. If the data $X$ are sampled from a sufficiently small neighborhood of a $d$-dimensional manifold with sufficiently mild curvature, and  $N$ is sufficiently large, then the relationship between $\mathbf{n}$ and $\mathbf{\tau}$ should be that of a quadratic polynomial in terms of tangential coordinates. This relationship takes the form $\mathbf{n}\approx\mathbf{th}$, where $\mathbf{h}$ is a matrix of quadratic coefficients\footnote{By our definition of $L$, we assume that linear dependence of $\mathbf{n}$ on $\mathbf{t}$ is negligible.} and $\mathbf{t}$ is a matrix of ones and quadratic monomials given by
\begin{equation}\label{eqn:quad_terms}
    \mathbf{t}=\left(\begin{array}{ccccccc}
        1 & \tau_1^{(1)}\tau_1^{(1)} & \ldots & \tau_1^{(1)}\tau_d^{(1)} & \tau_2^{(1)}\tau_1^{(1)} & \ldots & \tau_d^{(1)}\tau_d^{(1)} \\
        \vdots & \vdots & \ddots & \vdots & \vdots & \ddots & \vdots \\
        1 & \tau_1^{(N)}\tau_1^{(N)} & \ldots & \tau_1^{(N)}\tau_d^{(N)} & \tau_2^{(N)}\tau_1^{(N)} & \ldots & \tau_d^{(N)}\tau_d^{(N)}
    \end{array}\right),
\end{equation}
as in \cite{sritharan}. Regression by least squares solves for $\mathbf{h}$ by $\hat{\mathbf{h}}=\left(\mathbf{t}^\top\mathbf{t}\right)^{-1}\mathbf{t}^\top\mathbf{n}$. Instead of storing quadratic coefficients and constant terms in the matrix $\hat{\mathbf{h}}$, we can represent the relationship between tangential coordinate vector $\tv$ and normal coordinate vector $\nv$ as
\begin{equation}\label{eqn:quad_form_standard}
    \nv\approx\frac{1}{2}K \left(\tv\otimes\tv\right)+\vec{a},
\end{equation}
where $K$ is a matrix of quadratic coefficients, $\vec{a}$ is a vector of constants, and ``$\otimes$'' is the Kronecker product. Accordingly, points $\vec{x}$ in the ambient space obey the relationship
\begin{equation}\label{eqn:quad_form_general}
    \vec{x}\approx\frac{1}{2}M K \left(\tv\otimes\tv\right)+L \tv+\vec{x}_0,
\end{equation}
where $\vec{x}_0 = \vec{c} + \vec{a}$.

\subsection{Injectivity of na\"ive transition maps}\label{app:injectivity}
For a function $f:X\to Y$, let $\overline{f}:X\to X\times Y$ denote the function $\overline{f}:x\mapsto(x,f(x))$ (conventionally, $\overline{f}$ is called the \textit{graph} of $f$).
\begin{claim}\label{clm:lipschitz}
    Let $\mathcal{V},\mathcal{V}^\prime\subset\mathbb{R}^N$ be $d$-dimensional affine spaces with angle $\theta$ between them, and $\psi,\psi^\prime$ be the affine projections onto $\mathcal{V}$ and $\mathcal{V}^\prime$, respectively. If $F:\mathcal{V}\to\mathcal{V}^\perp$ is $K$-Lipschitz for $K<\tan\left(\frac{\pi}{2}-\theta\right)$, then $\psi^\prime\circ\overline{F}$ is injective.
    \begin{proof}
        Assume to the contrary that there exist distinct $a,b\in\mathcal{V}$ such that $(\psi^\prime\circ\overline{F})(a)=(\psi^\prime\circ\overline{F})(b)$. The map $\overline{F}$ is injective, so $\alpha:=\overline{F}(a),\beta:=\overline{F}(b)$ are distinct, and $\psi^\prime$ maps the line $l_{\alpha\beta}$ passing through $\alpha,\beta$ to a single point $p$. The line $l_{\alpha\beta}$ is the graph of some affine function $\varphi:\psi(l_{\alpha\beta})\to\mathcal{V}^\perp$ obeying the relation
        \begin{equation*}
            \frac{\left\lVert\varphi(x)-\varphi(y)\right\rVert_{\mathbb{R}^{N-d}}}{\left\lVert x-y\right\rVert_{\mathbb{R}^d}} \geq \tan\left(\frac{\pi}{2}-\theta\right)
        \end{equation*}
        for distinct $x,y\in\psi(l_{\alpha\beta})$. So
        \begin{equation*}
            \frac{\left\lVert F(a)-F(b)\right\rVert_{\mathbb{R}^{N-d}}}{\left\lVert a-b\right\rVert_{\mathbb{R}^d}} \geq\tan\left(\frac{\pi}{2}-\theta\right),
        \end{equation*}
        which contradicts our assumption that $F$ is $K$-Lipschitz.
    \end{proof}
\end{claim}

\subsection{Convergence of learned \atlas to differentiable atlas} \label{sec:app_atlas_learn_convergence}

\begin{claim}
    Let $\iota: \mathcal{M} \hookrightarrow \mathbb{R}^D$ be a smooth embedding of a manifold $\mathcal{M}$ into $\mathbb{R}^{D}$. We assume that $\iota$ is smooth and that $\mathcal{M}$ has compact closure in $\mathbb{R}^D$. We further assume that $\mathcal{M}$ is $\delta$-discernably embedded into $\mathbb{R}^D$ for $\delta > 0$ (i.e., for all $\vec{x} \in \operatorname{Im}\iota$, the intersection of the $D$-dimensional ball of radius $\delta$ centered at $\vec{x}$ is homotopic to a point (e.g.,~\cite{munkres}). Lastly, we assume that $\mathcal{M}$, as a submanifold of $\mathbb{R}^D$, has $\kappa$-uniformly bounded sectional curvature. We use $\mathfrak{g}$ to denote the Riemannian metric $\mathcal{M}$ inherits as a submanifold of $\mathbb{R}^D$.

    Say we have a measure $\mu$ on $\mathcal{M}$ that is measurable with respect to the uniform measure on $\mathcal{M}$. As the number of points $n$ sampled from $\mu$ approaches infinity, the discrepancy
    $$
    \| \varphi_i(\vec{\xi}) - \varphi_j(\psi_{ij}(\vec{\xi}))) \|_{\mathbb{R}^D}
    $$
    goes to zero for all $\vec{\xi} \in \mathcal{V}_{ij}$ whenever $\mathcal{V}_{ij}$ is defined.
\end{claim}

\begin{proof}
    The proof consists of four steps. Together, they allow us to relate the number $n$ of points sampled from $\mu$ to both the number of charts $k$ and the smallest radii $r_0$ of random balls centered at $\vec{x}_i$ to cover $\mathcal{M}$ with high probability. The proof is complete in establishing a relationship between $n$, $k$, and $r_0$ that results in vanishing discrepancy in probability.
    \begin{enumerate}
        \item First, we note that $\iota$ induces a metric space structure on $\mathcal{M}$, with distance $d_{\mathfrak{g}}$ given by the length of geodesics. Note that, for $\vec{x}, \vec{y} \in \mathcal{M}$, the distance $d_{\mathfrak{g}}\left(\vec{x}, \vec{y}\right)$ is bounded below by $\left\lVert \vec{x} - \vec{y} \right\rVert_2$. By compactness of the closure of $\mathcal{M}$, the expected, if we take a fixed radius $r_0 > 0$ and a probability measure $\mu$ supported a.e. w.r.t. the uniform measure on $\mathcal{X}$, the expected number of $d_{\mathfrak{g}}$-balls of radius $r_0$ with centers sampled from $\mu$ needed to cover $\mathcal{X}$ is finite. Further, the variance of this number is bounded in Proposition 2.1 from~\cite{aldous_2022}. In this way, as the number of balls generated by the process grows, the probability $p(n, \mu)$ that the closure of $\mathcal{M}$ is not covered vanishes.

        \item By the above, for any positive $r_0 < \delta$, there exist $n > k > 0$ with $X = \left\{\vec{x}_1, \ldots, \vec{x}_n\right\}, \vec{x}_i \sim \mu$ and $\overline{X} \subset X, \lvert \overline{X}\rvert = k$ selected using a $k$-medoids algorithm variant such that the nearest distance from point $\vec{x}_i \in X$ to $\vec{m}_j \in \overline{X}$ is less than $r_0$ in expectation. with probability $p(k, \mu)$ that goes to zero as $n$ goes to infinity and $k$ goes to infinity more slowly than $n$. Specifically, we require $k$ to grow with respect to $n$ such that there are enough points in each subset $S_i$ of the $k$-medoids partition that the quadratic fit learned in \atlaslearn is fully determined with high probability for each subset. In this way, not only are the quadratic fits always determined, but the radius of each $S_i$ is bounded above by $r_0 < \delta$.

        \item By the $\kappa$-uniform bound on sectional curvature, condition i) in Section 2.1 of~\cite{little} is met for sufficiently small $r_0 < \delta$. Therefore, there is a sufficiently small $r_0 < \delta$ and sufficiently high $n$ such that points within each $S_i$ satisfy Assumption I of Section 3.3 of~\cite{little}. In this way, a local tangent plane approximation $L$ can be learned according to Theorem 1 of~\cite{little}.
        \item By 3), we have that $\mathcal{V}_i$ and $\varphi_i$ are well defined for each $S_i$. Let $\mathcal{V}_j$ be such that $\mathcal{V}_{ij}$ is well defined, and let $\vec{\xi}$ be some element of $\mathcal{V}_{ij}$. Let $\mathcal{E}_{ij}$ denote the intersection of the MVEEs $\mathcal{E}_i$ and $\mathcal{E}_j$ defined by $A_i, \vec{b}_i, c_i$ and $A_j, \vec{b}_j, c_j$, respectively. By definition, $\varphi_i(\vec{\xi})$ is an $O(r_0^3)$ approximation of a point $\vec{x} \in \mathcal{E}_{ij}$, and $\varphi_j\left(\psi_{ij}(\vec{\xi})\right)$ is an $O(r_0^3)$ approximation of a point $\vec{y} \in \mathcal{E}_{ij}$. Becaues $\vec{x}$ and $\vec{y}$ both belong to $\mathcal{E}_{ij}$, the distance $\left\lVert \vec{x} - \vec{y} \right\rVert_2$ is of order $O(r_0^2)$. By the triangle inequality, the following series of deductions gives us the approximation order of the discrepancy between $\varphi_i(\vec{\xi})$ and $\varphi_j(\psi_{ij}(\vec{\xi}))$.
        \begin{align*}
            \left\lVert \varphi_i(\vec{\xi}) - \varphi_j\left(\psi_{ij}(\vec{\xi})\right) \right\rVert_2 &\leq \left\lVert \varphi_i(\vec{\xi}) - \vec{x} \right\rVert_2 + \left\lVert \vec{x} - \vec{y} \right\rVert_2 + \left\lVert \varphi_j\left(\psi_{ij}(\vec{\xi})\right) - \vec{y} \right\rVert_2 \\
            &= O(r_0^3) + O(r_0^2) + O(r_0^3) \\
            &= O(r_0^2)
        \end{align*}
        Taking the limit as $r_0$ goes to zero, and scaling $n$ and $k$ according to 2) above, completes the proof.        
    \end{enumerate}
\end{proof}

\subsection{Time complexity of \atlas primitives constructed via \atlaslearn}\label{sec:atlas_complexity}

In this section, we discuss the time complexity of base methods (i.e., the ones listed in Algorithm~\ref{spec:atlas_graph} implemented on \atlas objects, as implemented in one or more of our experiments. Methods used only the Grassmann experiments discussed in Section~\ref{sec:grassmann} are discussed separately in Section~\ref{app:grassmann}. Similarly, methods used only for computations on the Klein bottle---namely, length-adjusted retraction updates and logarithms on \atlas instances learned by \atlaslearn---these are discussed separately in Section~\ref{app:klein}.

\subsubsection{\texttt{in\_domain}}\label{sec:atlas_in_domain}

In \atlaslearn objects, the \texttt{in\_domain} method evaluates the expression
\begin{equation}
    \texttt{in\_domain}(\vec{\xi}) = \bigg[\varphi(\vec{\xi})^\top A \varphi(\vec{\xi}) + \vec{b}^\top \varphi(\vec{\xi}) + c  \leq 0\bigg].
\end{equation}
Naive evaluation of this expression is done in $O(D^2)$ time. This complexity be improved upon by leveraging Equation~(\ref{eqn:quad_form_general}); namely, we define
\begin{gather*}
    A_4 = K^\top M^\top A M K, \hspace{2cm} A_3 = L^\top A M K, \hspace{2cm} A_2 = L^\top A L, \\
    \vec{b}_2 = K^\top M^\top A \vec{x}_0 + \frac{1}{2} K M \vec{b}, \hspace{2cm} \vec{b}_1 = 2 L^\top A \vec{x}_0  + L^\top \vec{b}, \hspace{2cm} c_0 = \vec{x}_0^\top A \vec{x}_0 + \vec{b}^\top \vec{x}_0 + c,
\end{gather*}
and we implement the computation as
\begin{equation}
    \texttt{in\_domain}(\vec{\xi}) = \bigg[ \xikronsqr^\top A_4 \xikronsqr + \vec{\xi}^\top A_3 \xikronsqr + \vec{\xi}^\top A \vec{\xi} + \vec{b}_2^\top \xikronsqr + \vec{b}_1^\top \vec{\xi} + c_0 \leq 0 \bigg],
\end{equation}
where $\xikronsqr = \vec{\xi} \otimes \vec{\xi}$. In this way, the complexity becomes $O(d^4)$, which is an improvement in the case that $d$ is of order $o(\sqrt{D})$, at the cost of storing more arrays than $A$, $\vec{b}$, and $c$ per chart.

\subsubsection{\texttt{identify\_new\_chart}}\label{sec:atlas_identify}
When a retraction $\texttt{Ret}$ takes a point $\vec{\xi} \in \mathcal{V}_i$ along a tangent vector $\vec{\tau} \in T_{\vec{\xi}}\mathcal{V}_i$ to or past its boundary, identifying the chart to which $\texttt{Ret}_{\vec{\xi}}(\vec{\tau})$ best belongs can be assessed using different heuristics, each possessing different tradeoffs. Using the $\vec{x}_i$ learned from the local quadratic approximation in each $\mathcal{V}_i$, locally sensitive hashing can be used to identify the $k$ nearest charts to query for constant $k$. The goodness of fit of a new chart $\vec{V}_j$ can be assessed how well a point lies within the MVEE of that chart, as defined by $A_i$, $\vec{b}_i$, and $c_i$; this can be done after mapping $\texttt{Ret}_{\vec{\xi}}(\vec{\tau})$ to the ambient space in $O(D^2)$ time, or within $\mathcal{V}_j$ in $O(d^4)$ time. Alternatively, the goodness of fit can be assessed by the error of the local quadratic approximation of $\mathcal{V}$, which can be done with the same time complexities for ambient and $\mathcal{V}_j$ evaluation. Evaluating $\texttt{identify\_new\_chart}$, then, can be done in $O(cD^2 + cD d^2)$ or $O(cd^4)$ time.

\subsubsection{The coordinate chart map $\varphi$}\label{sec:atlas_coord_chart}
Chart coordinates $\vec{\xi} \in \mathcal{V}_i$ are mapped by $\varphi$ into the ambient space by evaluating the righthand side of Approximation~\ref{eqn:quad_form_general}, which is done in $O(D d^2)$ time.

\subsubsection{\texttt{D\_chart}}\label{sec:atlas_d_chart}

If a tangent vector $\tau \in T_{\texttt{chart}_i}\text{im}\texttt{chart}_i$ is represented in the ambient space explicitly, implementing the retraction $R_{\vec{p}}$ from Equation~(\ref{eq:retraction}) takes $O(D d) + K(D, d) + \tilde{K}(D, d)$ time, where $D$ is the ambient dimension, $d$ is the dimension of the \atlas object, $K(D, d)$ is the time complexity of computing $\texttt{D\_chart}_i$, and $\tilde{K}(D, d)$ is time complexity of inverting $\texttt{D\_chart}_i$. When learned using \atlaslearn, an \atlas object implements $\texttt{D\_chart}_i$ as
\begin{equation}
    \texttt{D\_chart}_i(\vec{\xi}) = M_i K_i \left(\vec{\xi} \otimes \vec{1}\right) + L
\end{equation}
when wanting to act on tangent vectors $\vec{v} \in T_{\varphi_i(\vec{\xi})}\mathcal{M}$ in ambient coordinates, but is the identity for tangent vectors $\vec{\tau} \in T_{\vec{\xi}}\mathcal{V}_i$ in chart coordinates (see Section~\ref{app:quad_approx_pt_cloud}). When acting on tangent vector $\vec{v} \in T_{\varphi_i(\vec{\xi})}\mathcal{M}$ in ambient coordinates\footnote{Here, we assume the product $M_i K_i$ is cached.}, $K(D, d) = O(D^2 (D - d))$, and $\tilde{K}(D, d) = O(D d^2)$\footnote{$\tilde{K}(D, d) = O(D d^2)$ here assuming the Moore-Penrose inverse is computed by singular value decomposition (SVD)~\cite{li_2019}.}. This is the time complexity relevant for implementing $\texttt{D\_chart}_i$ in the Dynamo vector field integration (Section~\ref{sec:bio}). When a tangent vector is represented implicitly by an element of the tangent space $T_{\vec{\xi}}\mathcal{V}$, then $\texttt{D\_chart}(\vec{\xi})$ is an implicit identity operation, and the retraction takes $O(d)$ time to compute for the vector addition. This representation of the tangent vector as an element of $T_{\vec{\xi}}\mathcal{V}$ is done in the RPB experiment (Section~\ref{sec:classification_patches}).

When there are too few samples for \atlaslearn to determine a unique quadratic fit within a neighborhood, but enough to determine a linear approximation, we learn $L$ and $\vec{x}_0$ as described in Algorithm~\ref{subsec:atlaslearn}, but set the trilinear form $K$ to the zero map. This is the case in running \atlaslearn with 30 charts on the Dynamo data in Section~\ref{sec:bio}. In this scenario, the runtime $K(D, d)$ of $\texttt{D\_chart}$ becomes $O((D - d)d)$, as the quadratic term is always zero.

\subsubsection{The transition map $\psi_{ij}$}
\label{sec:atlas-transition}
The transition map $\psi_{ij}$ is evaluated as the equation
\begin{equation}\label{eq:transition_map}
    \psi_{ij}(\vec{\xi}) = L_j^\top \left[\frac{1}{2}M_i K_i\left(\vec{\xi}\otimes \vec{\xi}\right) + \left(L_i \vec{\xi} + \vec{x}_{i}\right) - \vec{x}_{j}\right].
\end{equation}
When the products $L_j^\top M_i K_i$, $L_j^\top L_i$, and $L_j^\top \left(\vec{x}_i - \vec{x}_j\right)$ are precomputed, this evaluation can be done in $O(d^3)$ time. At first computation, the overhead times for computing these products are $O(D d (D - d))$, $O(D d^2)$, and $O(D d)$, respectively. When $L_j^\top M_i K_i$ is known to be the zero matrix, as may happen when there are too few points for a quadratic fit in a neighborhood, the time complexity reduces to $O(d^2)$.

\subsubsection{\texttt{transition\_vector}}\label{sec:atlas-transition-vector}
To map a tangent vector $\vec{\tau}$ from $T_{\vec{\xi}} \mathcal{V}_i$ to $T_{\vec{\eta}} \mathcal{V}_j$ for any $\vec{\eta} \in \mathcal{V}_j$, we evaluate

\begin{equation}\label{eq:transition_vector}
    \texttt{transition\_vector}(\vec{\xi}, \vec{\eta}, \vec{\tau}) = L_j^\top \left(M_i K_i \left(\vec{\xi} \otimes \vec{1}\right) + L_i\right) \vec{\tau}.
\end{equation}

If $L_j^\top M_i K_i$ and $L_j^\top L_i$ are precomputed, as discussed in Section~\ref{sec:atlas-transition}, this computation can be done in $O(d^3)$ time in the case of nonzero $K_i$ and $O(d^2)$ time otherwise.

\subsection{Approximating the Riemannian Retraction Logarithm}\label{sec:riem_log}

\begin{algorithm}[t]
    \caption{Approximate Riemannian logarithm}\label{alg:riem_log}
    \begin{algorithmic}
        \Require Dense subgraph $G_{\delta, \epsilon}=(V\subset\mathcal{M},E\subset V\times V)$\Comment{ See Sec.\ref{sec:app_naive_dist}}
        \Require Weight function $W:E\to\mathbb{R}$ given by na\"ive approx. dist. \Comment{Sec.~\ref{sec:app_naive_dist}}
        \Require Quad. approx. matrices $M_i,K_i,L_i,\vec{c}_i$ for all charts $i$
        \Require basepoint $\vec{\xi}$ in chart $i$
        \Require point $\vec{\xi}^*$ in chart $i^*$, of which to take Riemannian logarithm
        \State $\vec{x}\gets\frac{1}{2}M_iK_i\left(\vec{\xi}\otimes\vec{\xi}\right)+L_i\vec{\xi}+\vec{c}_i$ \Comment{Ambient representation of $\vec{\xi}$}
        \State $\vec{x}^*\gets\frac{1}{2}M_{i^*}K_{i*}\left(\vec{\xi}\otimes\vec{\xi}\right)+L_{i^*}\vec{\xi}+\vec{c}_{i^*}$ \Comment{Ambient representation of $\vec{\xi}^*$}
        \State $\vec{v}\gets\text{point in }V\text{ closest to }\vec{x}$
        \State $\vec{v}^*\gets\text{point in }V\text{ closest to }\vec{x}^*$
        \State $\left(\vec{v},\vec{v}_1,\ldots,\vec{v}_j,\vec{v}^*\right)\gets\text{shortest path from }\vec{v}\text{ to }\vec{v}^*\text{ in }(V,E,W)$
        \Comment{$\vec{\xi}^\prime$ stores the cumulative Riem. log to be returned}
        \Comment{$\vec{\xi}_{\text{dq}}$ is the next point of which to compute the Riem. log locally; ``dq''$=$``dequeue''}
        \State $\xi^\prime\gets\vec{0}$ \Comment{Initialize Riem. log. to zero}
        \State $\vec{\xi}_{\texttt{dq}}\gets\vec{\xi}^*$ \Comment{First point of which to take local log. is $\vec{\xi}^\prime$}
        \State $i_\texttt{dq}\gets\text{chart index of }\vec{\xi}^*$
        \For{$\tilde{v}\gets\vec{v}^*,\vec{v}_j,\ldots,\vec{v}_1,\vec{v}$}
            \State $\tilde{i}\gets\text{chart index of }\tilde{v}$
            \State $\tilde{\xi}\gets\text{representation of }\tilde{v}\text{ in chart }\tilde{i}$
            \If{$i_\texttt{dq}=\tilde{i}$}\Comment{If both points are in the same chart}
                \State $\vec{\xi}^\prime\gets\vec{\xi}^\prime+\left(\vec{\xi}_\texttt{dq}-\tilde{\xi}\right)$ \Comment{Increment cumulative log. by local log.}
                \State $\vec{\xi}_\texttt{dq}\gets\tilde{\xi}$ \Comment{The base of the local log. becomes operand of next local log.}
            \Else
                \Comment{We assume by density of $G_\text{NN}$ that there is a valid representative in chart $\tilde{i}$}
                \State $\vec{\xi}_\texttt{dq}\gets L_{\tilde{i}}^\top\left(\frac{1}{2}M_{i_\texttt{dq}}K_{i_\texttt{dq}}\left(\vec{\xi}_\texttt{dq}\otimes\vec{\xi}_\texttt{dq}\right)+L_{i_\texttt{dq}}\vec{\xi}_\texttt{dq}+\vec{c}_{i_\texttt{dq}}\right)$ \Comment{Rep. of $\vec{\xi}_\texttt{dq}$ in chart $\tilde{i}$}
                \State $\vec{\xi}^\prime\gets L_{\tilde{i}}^\top\left(\frac{1}{2}M_{i^\prime}K_{i^\prime}\left(\vec{\xi}^\prime\otimes\vec{\xi}^\prime\right)+L_{i^\prime}\vec{\xi}^\prime+\vec{c}_{i^\prime}\right)$ \Comment{Rep. of $\vec{\xi}^\prime$ in chart $\tilde{i}$}
                \State $\vec{\xi}^\prime\gets\vec{\xi}^\prime+\left(\vec{\xi}_\texttt{dq}-\tilde{\xi}\right)$ \Comment{Increment cumulative log. by local log.}
                \State $i_\texttt{dq}\gets\tilde{i}$
            \EndIf
        \EndFor
        \State $\vec{\xi}^\prime\gets\vec{\xi}^\prime-\vec{\xi}$ \Comment{Lastly, increment log. by $\mathbf{Log}^\mathbf{Ret}_{\vec{\xi}}\left(\vec{\xi}^\prime\right)$}
        \State\Return $\vec{\xi}^\prime$
    \end{algorithmic}
\end{algorithm}

When performing variants of Riemannian gradient descent on a loss function determined by manifold-valued data \cite{hosseini_and_sra}, the contribution of each datum to the Riemannian gradient of the loss function involves computing retraction logarithms. For the quasi-Euclidean retraction on \atlas, retraction logarithms can be computed easily. Specifically, for points $\vec{\xi}_1, \vec{\xi}_2$ belonging to the same compressed chart $\mathcal{V}$, the retraction logarithm of $\vec{\xi}_2$ at $\vec{\xi}_1$, or $\mathbf{Log}^{\mathbf{Ret}}_{\vec{\xi}_1}(\vec{\xi}_2)$, is simply $\vec{\xi}_2-\vec{\xi}_1$; for points $\vec{\xi}_1, \vec{\xi}_2$ that do not belong to the same compressed chart, we approximate $\mathbf{Log}^{\mathbf{Ret}}_{\vec{\xi}_1}(\vec{\xi}_2)$ using the shortest path between $\vec{\xi}_1$ and $\vec{\xi}_2$ on the nearest neighbor graph $G_{\delta,\epsilon}$ of densely sampled points from \atlas structure (Sec.~\ref{sec:app_naive_dist}), by iteratively summing and vector-transporting the edgewise retraction logarithms from the endpoint of the path back to the start (Algorithm~\ref{alg:riem_log}).

\section{Methods for performing online Grassmann Fr\'echet mean estimation}\label{app:grassmann}

Experiments were performed using one core on an Intel Xeon Gold-6448Y processor, 100GB RAM, and the Linux 4.18.0-305.3.1.el8.x86\_64 operating system.

\subsection{An \atlas representation of the (n,k)-Grassmannian}\label{sec:grass_example}

The \atlas structure is constructed from a conventional atlas of $\mathbf{Gr}_{n,k}$ derived from a cell complex presented by Ehresmann~\cite{ehresmann_1934}, which we adapt to our notation for convenience. Our \atlas has charts from the Ehresmann atlas and also permits \textit{ad hoc} creation of new coordinate charts centered at any point in the manifold. We use \textit{ad hoc} chart creation to maintain proximity to the origin in compressed charts and, hence, accuracy of quasi-Euclidean updates, which results in fast online learning.

In addition to the Ehresmann atlas, there exists another canonical representation of the Grassmann manifold as a quotient of the Stiefel matrices~\cite{bendokat_2011}. We find the Ehresmann atlas simpler to present, but also use the Stiefel construction (Algorithm~\ref{alg:grass_quasi_euclidean} and Fig.~\ref{fig:first_order_subroutines}) to enable direct comparisons with existing Riemannian optimization approaches on the Grassmann manifold that use it, such as Pymanopt \cite{pymanopt} and GiFEE \cite{chakraborty_gifme}. 

\subsection{Coordinate charts of the Ehresmann atlas}\label{sec:ehresmann}
We begin with some intuition for the Ehresmann atlas construction. The $(n,k)$-Grassmannian can be understood as the manifold of $n\times n$ orthogonal projection matrices of rank $k$. There are $\binom{n}{k}$ such matrices whose entries are all zero, save for exactly $k$ diagonal entries that are equal to one. These matrices are in one-to-one correspondence with sets of $k$ fixed integers $i_1,\ldots,i_k$ satisfying $1\leq i_1<\ldots<i_k\leq n$, and therefore with the permutations\footnote{For finite $S\subset\mathbb{R}$, the notation $\min_{(m)}S$ denotes the $m$th smallest element of $S$.}
\begin{align}\label{eqn:permutation_indices}
    \pi_{i_1,\ldots,i_k}:\{1,\ldots,n\}&\to\{1,\ldots,n\} \\
    j&\mapsto\left\{\begin{array}{lr}
        i_j, & j\leq k \\
        \min_{(j-k)}\left(\left\{1,\ldots,n\right\}\setminus\left\{i_1,\ldots,i_k\right\}\right), & j>k \nonumber
    \end{array}\right..
\end{align}
Let $P_{i_1,\ldots,i_k}$ be the $n\times n$ permutation matrix corresponding to $\pi_{i_1,\ldots,i_k}$. The aforementioned $\binom{n}{k}$ matrices take the form
\begin{equation*}
    P_{i_1,\ldots,i_k}\left(\frac{I_k}{\mathbf{0}_{n-k,k}}\right)\left(\frac{I_k}{\mathbf{0}_{n-k,k}}\right)^\top P_{i_1,\ldots,i_k}^\top.
\end{equation*} 
These $\binom{n}{k}$ points are the centerpoints of each coordinate chart in the Ehresmann atlas. Further, for every $P\in\mathbf{Gr}_{n,k}$, there exist a permutation $\pi_{i_1,\ldots,i_k}$ and a matrix $A\in\mathbb{R}^{(n-k)\times k}$ such that $P=P_{i_1,\ldots,i_k}\left(\frac{I_k}{A}\right)\left(P_{i_1,\ldots,i_k}\left(\frac{I_k}{A}\right)\right)^\dag$, where $\dag$ denotes the Moore-Penrose pseudoinverse.

For the coordinate charts in the Ehresmann atlas 
let  $\mathbf{colproj}:\mathbb{R}^{n\times k}\to\mathbb{R}^{n\times n}$ be the map that takes a matrix $A$ to the orthogonal projector onto the column space of $A$. 
The coordinate chart $\left(\mathcal{U}_{i_1,\ldots,i_k},\mathcal{V}_{i_1,\ldots,i_k},\varphi_{i_1,\ldots,i_k}\right)$ is then defined as:
\begin{itemize}
    \setlength\itemsep{0em}
    \item $\varphi_{i_1,\ldots,i_k}^{-1}:A\mapsto\mathbf{colproj}\left(P_{i_1,\ldots,i_k}\left(\frac{I_k}{A}\right)\right)$
    \item $\mathcal{V}_{i_1,\ldots,i_k}=\mathbb{R}^{(n-k)\times k},$ \quad and \quad $\mathcal{U}_{i_1,\ldots,i_k}=\text{im}\left(\varphi_{i_1,\ldots,i_k}^{-1}\right).$
\end{itemize}
For convenience, we denote the coordinate chart corresponding to the identity permutation as $\left(\mathcal{U}_0,\mathcal{V}_0,\varphi_0\right)$. In Claims~\ref{clm:grass_dist_0}~and~\ref{clm:grass_dist_other}~(Appendix \ref{app:misc_grass_comp}), we demonstrate that a point $\vec{\xi}$ in a compressed chart $\mathcal{V}$ of the Ehresmann atlas is closer to the center of another chart than to the center of $\mathcal{V}$ if any coordinate of $\vec{\xi}$ exceeds one (Section~\ref{app:grass_atlas_graph}). 
We use this result to determine when to generate new \emph{ad hoc} charts. The creation of \textit{ad hoc} coordinate charts allows for points to always be close to the origin within a coordinate chart, thereby reducing the error of quasi-Euclidean updates. In the case that $P\in\mathbf{Gr}_{n,k}$ is closer to the center of $\mathcal{V}$ than to the center of any Ehresmann chart, we say that $\mathcal{V}$ is the ``best'' Ehresmann chart for $P$.

\subsection{Ad hoc formation of coordinate charts}\label{sec:ad_hoc}
To perform \textit{ad hoc} chart formation, we begin with the case of generating charts centered at projection matrices belonging to $\mathcal{U}_0$. This case generalizes to projection matrices belonging to $\mathcal{U}_{i_1,\ldots,i_k}$ by simple conjugation by $P_{i_1,\ldots,i_k}$. The group $\mathbf{O}(n)$ of $n\times n$ orthogonal matrices acts transitively on $\mathbf{Gr}_{n,k}$ according to the action $Q:P\mapsto QPQ^\top$. Therefore, for all $A\in\mathcal{V}_0$, there exists an orthogonal matrix $Q_A$ such that
\begin{equation}\label{eqn:qa_action}
    \varphi^{-1}_0(A)=Q_A\left(\begin{array}{c|c}
        I_k & \mathbf{0}_{k, n-k} \\
        \hline
        \mathbf{0}_{n-k,k} & \mathbf{0}_{n-k,n-k}
    \end{array}\right)Q_A^\top.
\end{equation}

\begin{claim}\label{clm:Q_A}
    An orthogonal matrix $Q_A$ which satisfies Equation \ref{eqn:qa_action} is given by
    \begin{equation*}
        Q_A=\left(\begin{array}{c|c}
            \sqrt{\left(I_k+A^\top A\right)^{-1}} & -A^\top\sqrt{\left(I_{n-k}+AA^\top\right)^{-1}} \\
            \hline
            A\sqrt{\left(I_k+A^\top A\right)^{-1}} & \sqrt{\left(I_{n-k}+AA^\top\right)^{-1}}
        \end{array}\right).
    \end{equation*}
    \begin{proof}
        See Section \ref{app:misc_grass_comp}.
    \end{proof}
\end{claim}

Claim \ref{clm:Q_A} allows us to 
change coordinate charts so that quasi-Euclidean updates approximate the exponential map more accurately, as done in Algorithm \ref{alg:grass_quasi_euclidean}. While this specific method 
only works for points in the Ehresmann chart $\varphi_0: \mathcal{V}_0 \to \mathcal{U}_0$, it generalizes to charts $\varphi_{i_1,\ldots,i_k}: \mathcal{V}_{i_1,\ldots,i_k} \to \mathcal{U}_{i_1,\ldots,i_k}$ by replacing the action in Equation \ref{eqn:qa_action} with the action
\begin{equation}\label{eqn:qa_action_general}
    \varphi^{-1}_{i_1,\ldots,i_k}(A)=Q_AP_{i_1,\ldots,i_k}\left(\begin{array}{c|c}
        I_k & \mathbf{0}_{k,n-k} \\
        \hline
        \mathbf{0}_{n-k,k} & \mathbf{0}_{n-k,n-k}
    \end{array}\right)P_{i_1,\ldots,i_k}^\top Q_A^\top.
\end{equation}
This creates coordinate charts $\varphi_{A,i_1,\ldots,i_k}: \mathcal{V}_{A,i_1,\ldots,i_k} \to \mathcal{U}_{A,i_1,\ldots,i_k},\varphi_{A,i_1,\ldots,i_k}$ for all $A\in\mathbb{R}^{(n-k)\times k}$ and all $P_{i_1,\ldots,i_k}$, with the transition map $\psi_{A,i_1,\ldots,i_k\to A^\prime,j_1,\ldots,j_k}$ taking $B\in\mathcal{V}_{i_1,\ldots,i_k}$ to
\begin{equation}\label{eqn:grass_transition_map}
    \left(\mathbf{0}_{n-k,k}\mid I_{n-k}\right)R\left(\left(I_k\mid\mathbf{0}_{k,n-k}\right)R\right)^{-1},
\end{equation}
where $R=P_{j_1,\ldots,j_k}^\top Q_{A^\prime}^\top Q_AP_{i_1,\ldots,i_k}\left(\frac{I_k}{B}\right)$.
Within any coordinate chart $\mathcal{V}_{A,i_1,\ldots,i_k}$, transition maps are invoked whenever any element of the coordinates $B$ exceeds one~(Section~\ref{app:grass_atlas_graph}), transitioning into a chart centered at $\varphi_{A,i_1,\ldots,i_k}\left(B\right)$.

\subsection{Transition maps in the \atlas representation of the Grassmannian}\label{app:grass_atlas_graph}

Here, we explain when we invoke transition maps between charts on the \atlas representation of the Grassmann manifold $\mathbf{Gr}_{n,k}$ given in Sec.~\ref{sec:grass_example}. For a given Ehressman chart, the adjacent charts are determined by the permutation Eq.~\ref{eqn:permutation_indices}, one existing for every dimension of the manifold. The condition under which we change charts within this \atlas representation is given in Claim~\ref{clm:when_transition}. For this reason, a transition boundary exists for each coordinate such that, if the coordinate exceeds an absolute value of one, the transition map is invoked.

\begin{claim}\label{clm:when_transition}
    An element $A$ of a fixed Ehressman chart is closer to the center of the chart than to the center of any other Ehressman chart if all of the elements of $A$ have absolute value less than one.
\end{claim}
\begin{proof}
    The claim is proved for the Ehressman chart $\varphi_0: \mathcal{V}_0 \to \mathcal{U}_0$ by Claims~\ref{clm:grass_dist_0} and \ref{clm:grass_dist_other}. For the remaining Ehressman charts, the Claim is proved by observing we can conjugate elements of $\mathcal{U}_0$ by the permutation matrices used to define the remaining Ehressman charts.
\end{proof}

In the online Fr\'echet mean experiment, transition maps between coordinate charts are implemented according to Algorithm~\ref{alg:grass_trans_map}, whose runtime scales with complexity $O(n^3+n^2k+nk^2+k^3)$ but is constant with respect to the number of compressed charts traversed. Since the quasi-Euclidean updates have $O(nk)$ time complexity, while the first-order update schemes for GiFEE, MANOPT, and MANOPT-RET have $\Omega(nk)$ time-complexity, the ATLAS framework can outspeed the other schemes when it can make infrequent use of high-cost chart transitions.

\subsection{Miscellaneous subroutines}\label{app:grass_misc_subroutines}
The online Fr\'echet mean estimation on the \atlas structure(Algorithm~\ref{alg:grass_quasi_euclidean}) depends on the subroutines \texttt{Atlas\_Grassmann\_identify\_chart} (Algorithm~\ref{alg:grass_identify}), which identifies the closest Ehressman chart,  and \texttt{Atlas\_Grassmann\_ingest\_matrix} (Algorithm~\ref{alg:grass_ingest}), which gives the representation of  matrix in the current coordinate chart of the \atlas.

\begin{figure}[t]
    \centering
    \fbox{\begin{subfigure}{0.45\columnwidth}
        \caption{\texttt{GifeeLog}}
        \begin{algorithmic}
            \footnotesize
            \Require $\mathcal{X},\mathcal{Y}\in\mathbf{Gr}_{n,k}$
            \Require $X\in\mathbb{R}^{n\times k},\mathbf{colproj}(X)=\mathcal{X}$
            \Require $Y\in\mathbb{R}^{n\times k},\mathbf{colproj}(Y)=\mathcal{Y}$
            \State $A\gets\left(I-X\left(X^\top X\right)^{-1}X^\top\right)Y\left(X^\top Y\right)^{-1}$
            \State $U,\Sigma,V^\top\gets\texttt{ThinSVD}(A)$
            \State \Return $U,\Sigma,V^\top$
        \end{algorithmic}
    \end{subfigure}}
    \fbox{\begin{subfigure}{0.45\columnwidth}
        \caption{\texttt{ManoptLog}}
        \begin{algorithmic}
            \footnotesize
            \Require $\mathcal{X},\mathcal{Y}\in\mathbf{Gr}_{n,k}$
            \Require $X\in\mathbb{R}^{n\times k},\mathbf{colproj}(X)=\mathcal{X}$
            \Require $Y\in\mathbb{R}^{n\times k},\mathbf{colproj}(Y)=\mathcal{Y}$
            \State $A\gets\left(I-X\left(X^\top X\right)^{-1}X^\top\right)Y\left(X^\top Y\right)^{-1}$
            \State $U,\Sigma,V^\top\gets\texttt{ThinSVD}(A)$
            \State $\Theta\gets\arctan\Sigma$
            \State \Return $U\Theta V^\top$
        \end{algorithmic}
    \end{subfigure}} \\
    \fbox{\begin{subfigure}{0.45\columnwidth}
        \caption{\texttt{GifeeExp}}
        \begin{algorithmic}
            \footnotesize
            \Require $\mathcal{X}\in\mathbf{Gr}_{n,k}$
            \Require $X\in\mathbb{R}^{n\times k},\mathbf{colproj}(X)=\mathcal{X}$
            \Require $U,\Sigma,V^\top$ from \texttt{GifeeLog}
            \Require iteration number $i>0$
            \State $\Theta\gets\arctan\Sigma$
            \State $\tilde{Z}\gets XV\cos(\Theta/i)+U\sin(\Theta/i)$
            \State $Z,R\gets$ QR decomposition of $\tilde{Z}$
            \State \Return $Z$
        \end{algorithmic}
    \end{subfigure}}
    \fbox{\begin{subfigure}{0.45\columnwidth}
        \caption{\texttt{ManoptExp}}
        \begin{algorithmic}
            \footnotesize
            \Require $\mathcal{X}\in\mathbf{Gr}_{n,k}$
            \Require $X\in\mathbb{R}^{n\times k},\mathbf{colproj}(X)=\mathcal{X}$
            \Require $L$ from \texttt{ManoptLog}
            \Require iteration number $i>0$
            \State $U,\Sigma,V^\top\gets\texttt{ThinSVD}(L/i)$
            \State $\tilde{Z}\gets XV\cos(\Sigma) V^\top+U\sin(\Sigma) V^\top$
            \State $Z,R\gets$ QR decomposition of $\tilde{Z}$
            \State \Return $Z$
        \end{algorithmic}
    \end{subfigure}} \\
    \fbox{\begin{subfigure}{0.45\columnwidth}
        \caption{\texttt{ManoptRet}}
        \begin{algorithmic}
            \footnotesize
            \Require $\mathcal{X}\in\mathbf{Gr}_{n,k}$
            \Require $X\in\mathbb{R}^{n\times k},\mathbf{colproj}(X)=\mathcal{X}$
            \Require $L$ from \texttt{ManoptLog}
            \Require iteration number $i>0$
            \State $U,\Sigma,V^\top\gets\texttt{ThinSVD}(X+L/i)$
            \State $\tilde{Z}\gets UV^\top$
            \State $Z,R\gets$ QR decomposition of $\tilde{Z}$
            \State \Return $Z$
        \end{algorithmic}
    \end{subfigure}}

    \caption{\textbf{Riemannian logarithm and retraction algorithms used by the non-ATLAS first-order update schemes in Sec.~\ref{sec:grassmann}.}}\label{fig:first_order_subroutines}
\end{figure}

\begin{algorithm}[t]
    \caption{\texttt{QA\_from\_A} ($Q_A$ for $A\in\mathcal{V}_0$ according to Claim \ref{clm:Q_A})}\label{alg:qa_from_a}
    \begin{algorithmic}
        \Require $A\in\mathbb{R}^{n-k,k}$
        \State $Q_A\gets\left(\begin{array}{c|c}
            \sqrt{\left(I_k+A^\top A\right)^{-1}} & -A^\top\sqrt{\left(I_{n-k}+AA^\top\right)^{-1}} \\
            \hline
            A\sqrt{\left(I_k+A^\top A\right)^{-1}} & \sqrt{\left(I_{n-k}+AA^\top\right)^{-1}}
        \end{array}\right)$ \Comment{$O(n^3+n^2k+nk^2+k^3)$ time}
        \State \Return $Q_A$
    \end{algorithmic}
\end{algorithm}

\begin{algorithm}[t]
    \caption{\texttt{Atlas\_Grassmann\_identify\_chart} (identify chart in Ehresmann atlas whose center is closest to $\mathbf{colspan}X$)}\label{alg:grass_identify}
    \begin{algorithmic}
        \Require $X\in\mathbb{R}^{n\times k}$, full rank
        \State $P\gets X\left(X^\top X\right)^{-1}X^\top$ \Comment{$P=\mathbf{colproj}X$; $O(n^2k+nk^2+k^3)$ time}
        \For{$j\in\{1,\ldots,k\}$} \Comment{$O(nk)$ time}
            \State $i_j\gets i\text{ such that }P_{ii}\text{ is }j\text{th largest diagonal entry of }P$
            \State $i_1,\ldots,i_k\gets i_1,\ldots,i_k$ in increasing order
        \EndFor
        \State \Return $i_1,\ldots,i_k$
    \end{algorithmic}
\end{algorithm}

\begin{algorithm}[t]
    \caption{\texttt{Atlas\_Grassmann\_ingest\_matrix} (ingest Grassmann element represented as full-rank matrix into Ehresmann chart)}\label{alg:grass_ingest}
    \begin{algorithmic}
        \Require $X\in\mathbb{R}^{n\times k}$, full rank
        \Require $1\leq i_1\leq\ldots\leq i_k\leq n$ specifying Ehresmann chart
        \State $X_U\gets$ restriction of $X$ to rows in $\{i_1,\ldots,i_k\}$ \Comment{$O(1)$ time}
        \State $X_L\gets$ restriction of $X$ to rows not in $\{i_1,\ldots,i_k\}$ \Comment{$O(1)$ time}
        \State $A\gets X_LX_U^{-1}$ \Comment{$A=\varphi_{i_1,\ldots,i_k}\left(\mathbf{colproj}\left(X\right)\right)$; $O(nk^2+k^3)$}
        \State \Return $A$
    \end{algorithmic}
\end{algorithm}

\begin{algorithm}[t]
    \caption{\texttt{Atlas\_Grassmann\_transition\_map} (transition map on Grassmann \atlas, used in Alg.~\ref{alg:grass_quasi_euclidean})}\label{alg:grass_trans_map}
    \begin{algorithmic}
        \Require $A\in\mathbb{R}^{(n-k)\times k}$, permutation indices $i_1,\ldots,i_k$
        \Require $Q_A$ \Comment{output by \texttt{QA\_from\_A}$(A)$, Alg.~\ref{alg:qa_from_a}}
        \State $P\gets\texttt{permutation\_matrix}\left(i_1,\ldots,i_k\right)$ \Comment{$O(n)$ time}
        \State $Y\gets Q_AP\left(\frac{I_k}{A}\right)$ \Comment{$O(n^3+n^2k)$ time}
        \State $i_1,\ldots,i_k\gets\texttt{ATLAS\_identify\_chart}\left(Y\right)$ \Comment{Alg.~\ref{alg:grass_identify}; $O(n^2k+nk^2+k^3)$ time}
        \State $\tilde{A}\gets\texttt{ATLAS\_ingest\_matrix}\left(Y,i_1,\ldots,i_k\right)$ \Comment{Alg.~\ref{alg:grass_ingest}; $O(nk^2+k^3)$ time}
        \State $\tilde{Q}_A\gets\texttt{QA\_from\_A}\left(\tilde{A}\right)$ \Comment{Alg.~\ref{alg:qa_from_a}; $O(n^3+n^2k+nk^2+k^3)$ time}
        \State $Q_A\gets P\tilde{Q}_AP^\top$ \Comment{$O(n^3)$ time}
        \State $A\gets\mathbf{0}_{n-k,k}$ \Comment{$O(1)$ time}
        \State \Return $A,Q_A$
    \end{algorithmic}
\end{algorithm}


\begin{algorithm}[h]
    \caption{\atlasfrechet \\
    Online Fr\'echet mean estimation on $\mathbf{Gr}_{n,k}$ using quasi-Euclidean updates on \atlas
    } \label{alg:grass_quasi_euclidean}
    \begin{algorithmic}[1]
        \Require Probability distribution $\mathcal{D}$ on $\mathbf{Gr}_{n,k}$
        \Require Fr\'echet stream of samples $X_1,X_2,\ldots\sim\mathcal{D}$, sampled as Stiefel matrices
        \State $i_1,\ldots,i_k\gets\texttt{ATLAS\_identify\_chart}\left(X_1\right)$ \Comment{Best Ehresmann chart (Sec. \ref{sec:ehresmann}); Alg. \ref{alg:grass_identify}}
        \State $A\gets\texttt{ATLAS\_ingest\_matrix}\left(X_1,i_1,\ldots,i_k\right)$ \Comment{$A=\varphi_{i_1,\ldots,i_k}\left(\mathbf{colproj}\left(X_1\right)\right)$; Alg. \ref{alg:grass_ingest}}
        \State $Q_A\gets I_n$ \Comment{$Q_A$ is used to define the map (\ref{eqn:qa_action})}
        \State $Q_{A,U}\gets \left(I_k\mid\mathbf{0}_{k,n-k}\right)^\top$ \Comment{restriction of $Q_A$ to columns in $\{i_1,\ldots,i_k\}$}
        \State $Q_{A,L}\gets \left(I_{n-k}\mid\mathbf{0}_{n-k,k}\right)^\top$ \Comment{restriction of $Q_A$ to columns not in $\{i_1,\ldots,i_k\}$}
        \State $n\gets 1$
        \While{streaming}
            \State $n\gets n+1$
            \State $\tilde{A}\gets Q_{A,L}^\top X_n\left(Q_{A,U}^\top X_n\right)^{-1}$ \Comment{$\tilde{A}=\varphi_{i_1,\ldots,i_k}\left(Q_A^\top\mathbf{colproj}\left(X_n\right)Q_A\right)$; Eq. (\ref{eqn:qa_action_general})} \label{line:grass_logarithm}
            \State $A\gets A+\left(\tilde{A}-A\right)/n$\Comment{Update online Fr\'echet mean estimator}
            \If{any entry $a$ of $A$ violates $\lvert a\rvert<1$}\Comment{change chart if necessary; Claim \ref{clm:grass_dist_0}}
                \State $A,Q_A\gets\texttt{Atlas\_Grassmann\_transition\_map}(A,i_1,\ldots,i_k,Q_A)$ \Comment{Alg. \ref{alg:grass_trans_map}}
                \State $Q_{A,U}\gets$ restriction of $Q_A$ to columns in $\{i_1,\ldots,i_k\}$
                \State $Q_{A,L}\gets$ restriction of $Q_A$ to columns not in $\{i_1,\ldots,i_k\}$
            \EndIf
        \EndWhile
        \State \Return $A,i_1,\ldots,i_k,Q_A$
    \end{algorithmic}
\end{algorithm}

\subsection{Non-uniform sampling of the Grassmann manifold by geodesic power scaling}\label{sec:geodesic_power_distribution}

Because of the lack of existing benchmark datasets, we introduce the  \emph{geodesic power distribution}
$\mathbf{GPD}(\mathcal \mathcal{X} , p)$ with Fr\'echet mean  $\mathcal X \in\mathbf{Gr}_{n,k}$ and scaling exponent $p > 1$, which inversely controls the entropy of the distribution (formally defined in Section~\ref{sec:geodesic_power_distribution}). The $\mathbf{GPD}$ is a natural, efficiently samplable distribution, which guarantees existence and uniqueness of the Fr\'echet mean $\mathcal{X}$.
For the Grassmann experiments, points for online sampling were generated by sampling randomly from the Grassmann manifold. To generate points from a distribution with a well-defined Fr\'echet mean, such as the uniform measure on balls of a fixed radius on the Grassmannian, we used the following method, which we call \textit{geodesic power scaling}. 

Fix $p>1$ and $\mathcal{X}\in\mathbf{Gr}_{n,k}$. Points are sampled from the \textit{geodesic power distribution} $\mathbf{GPD}(\mathcal{X}, p)$ as follows:
\begin{enumerate}
    \setlength\itemsep{0em}
    \item A point $\mathcal{Y}$ is sampled uniformly from $\mathbf{Gr}_{n,k}$.
    \item The Grassmann distance $\delta$ between $\mathcal{X}$ and $\mathcal{Y}$ is computed.
    \item A new point $\mathcal{Y}^\prime$ is computed as $\mathcal{Y}^\prime=\exp_\mathcal{X}\left(\left(\frac{\delta}{\delta_{\max}}\right)^p\log_\mathcal{X}\mathcal{Y}\right)$, where $\delta_{\max}$ is the largest possible Grassmann distance between two points on $\mathbf{Gr}_{n,k}$.\footnote{$\delta_{\max}=\frac{\pi}{2}\sqrt{\max\{k,n-k\}}$} 
\end{enumerate}

Whenever we sample from $\mathbf{GPD}(\mathcal{X},p)$, we assume that the sample takes the form of a matrix $X\in\mathbb{R}^{n\times k}$ such that $\mathcal{X}$ is the column space of $X$. Increasing the scaling exponent $p$ reduces the entropy of the distribution by concentrating probability around $\mathcal{X}$; as $p\to\infty$, $\mathbf{GPD}(\mathcal{X},p)$ approaches the Dirac delta based at $\mathcal{X}$. Note that the distribution is invariant under action by the special orthogonal group, as long as the action fixes $\mathcal{X}$, giving the distribution ``rotational symmetry'' about $\mathcal{X}$. Moreover, $\mathbf{GPD}(\mathcal{X},p)$ has unique population Fr\'echet mean $\mathcal{X}$ for all $p>1$.

A $\mathbf{GPD}$ satisfies the $L^2$-moment constraint for $p>1$; while it does not satisfy the support constraint, the probability density function $\mathbf{GPD}(\mathcal{X},p)$ is close to zero for most points $\mathcal{X}\in\mathbf{Gr}_{n,k}$ for sufficiently high $p$. For the sake of fair comparison, points are sampled from $\mathbf{GPD}\left(\mathcal{X},p\right)$ not as orthogonal projection matrices $P$, but as Stiefel matrices $X$ satisfying $P=XX^\top$. This is done by implementing the procedure of sampling from $\mathbf{GPD}\left(\mathcal{X},p\right)$ (Section~\ref{sec:geodesic_power_distribution}) in Pymanopt, which represents elements of $\mathbf{Gr}_{n,k}$ in terms of Stiefel matrices.

\subsection{Theoretical results for the Grassmann manifold}

\subsubsection{Ingesting columnspanning matrix into the \atlas representation of the Grassmannian}\label{sec:ingestion}
Say $X\in\mathbb{R}^{n\times k}$ has full rank. Thinking of $\mathbf{Gr}_{n,k}$ as the manifold of $n\times n$ orthogonal projection matrices of rank $k$, we know that the columnspace of $X$ is uniquely represented in $\mathbf{Gr}_{n,k}$ by $X\left(X^\top X\right)^{-1}X^\top$. Finding the chart to which $X$ belongs is tantamount to finding the ``centerpoint'' projection matrix to which $X\left(X^\top X\right)^{-1}X^\top$ is closest. This, in turn, is equivalent to finding the centerpoint projection matrix with which $X\left(X^\top X\right)^{-1}X^\top$ has the highest Frobenius inner product. This is accomplished by finding the $k$ largest diagonal entries of $X\left(X^\top X\right)^{-1}X^\top$, as demonstrated by the following series of deductions.
\begin{align*}
    \left\langle X\left(X^\top X\right)^{-1}X^\top,\sum_{j=1}^k\vec{e}_{i_j}\vec{e}_{i_j}^\top\right\rangle_{\text{Fr}}&=\mathbf{Tr}\left[X\left(X^\top X\right)^{-1}X^\top\sum_{j=1}^k\vec{e}_{i_j}\vec{e}_{i_j}^\top\right] \\
    &=\sum_{j=1}^k\mathbf{Tr}\left[X\left(X^\top X\right)X^\top\left(\vec{e}_{i_j}\vec{e}_{i_j}^\top\right)\right] \\
    &=\sum_{j=1}^k\left[X\left(X^\top X\right)^{-1}X^\top\right]_{i_ji_j}
\end{align*}
Note that the diagonal entries of $X\left(X^\top X\right)^{-1}X^\top$ are guaranteed to be nonnegative by the positive-semidefiniteness of $X\left(X^\top X\right)^{-1}X^\top$.

More generally, let $Q$ and $P$ be orthogonal matrices such that
\begin{equation*}
    Q\left(\begin{array}{c|c}
        I_k & \mathbf{0}_{k,n-k} \\
        \hline
        \mathbf{0}_{k,n-k} & \mathbf{0}_{n-k,n-k}
    \end{array}\right)Q^\top=P.
\end{equation*}
We ingest a full-rank matrix $X\in\mathbb{R}^{n\times k}$ into the chart centered at $P$ by the map
\begin{equation}\label{eqn:ingestion_general}
    X\mapsto X_LX_U^{-1},
\end{equation}
where $X_U,X_L$ are given by
\begin{equation*}
    X_U=\left(I_k\mid\mathbf{0}_{k,n-k}\right)Q^\top X,\hspace{2cm}X_L=\left(\mathbf{0}_{n-k,k}\mid I_{n-k}\right)Q^\top X.
\end{equation*}

\subsubsection{Distances on the Grassmann manifold} \label{app:misc_grass_comp}
\begin{claim}\label{clm:grass_dist_0}
    Let $\varphi_0$ be the Ehresmann coordinate chart map (Sec.~\ref{sec:grass_example}). For $t\in\mathbb{R}$, we define $V_t:=\varphi_0\left(t\vec{e}_{n-k}\vec{e}_k\right)=\left(\frac{I_k}{t\vec{e}_{n-k}\vec{e}_k^\top}\right)$. The Grassmann distance between $\mathbf{colproj}\left(V_t\right)$ and $\mathbf{colproj}(U)$, where $U:=\varphi_0\left(\mathbf{0}_{n-k,k}\right)=\left(\frac{I_k}{\mathbf{0}_{n-k,k}}\right)$, is equal to $\left\lvert\arctan t\right\rvert$.
    \end{claim}

\begin{proof}
        The square of the Grassmann distance $\mathbf{dist}_{\mathbf{Gr}}$ between two projection matrices $P,Q\in\mathbf{Gr}_{n,k}$ is given by the sum of squares of Jordan angles between their subspaces. From Lemma 5 in \cite{neretin}, if matrices $U_P,U_Q\in\mathbb{R}^{n\times k}$ satisfy $P=\mathbf{colproj}U_P$ and $Q=\mathbf{colproj}U_P$, the squares of cosines of the Jordan angles between $P$ and $Q$ are the eigenvalues of the matrix $\left(U_P^\top U_P\right)^{-1}U_P^\top U_Q\left(U_Q^\top U_Q\right)^{-1}U_Q^\top U_P$. Therefore, the following series of deductions holds.
        \begin{align*}\mathbf{dist}_{\mathbf{Gr}}\big(\varphi_0\left((t\vec{e}_k\vec{e}_k^\top)\right),\varphi_0\left(\mathbf{0}_{n-k,k}\right)\big)&=\sqrt{\mathbf{Tr}\left[\arccos\left(\sqrt{\left(U^\top U\right)^{-1}U^\top V_t\left(V_t^\top V_t\right)^{-1}V_t^\top U}\right)^2\right]} \\     &=\sqrt{\mathbf{Tr}\left[\arccos\left(\sqrt{\left(I_k+t^2\vec{e}_k\vec{e}_k^\top\right)^{-1}}\right)^2\right]} \\
            &=\sqrt{\mathbf{Tr}\left[\arccos\left(\sqrt{I_k-\frac{t^2}{1+t^2}\vec{e}_k\vec{e}_k^\top}\right)^2\right]} \\
            &=\sqrt{\mathbf{Tr}\left[\left(\begin{array}{c|c}
                \mathbf{0}_{k-1,k-1} & \mathbf{0}_{k-1,1} \\
                \hline
                \mathbf{0}_{1,k-1} & \arccos\left(\sqrt{\frac{1}{1+t^2}}\right)
            \end{array}\right)^2\right]} \\
            &=\left\lvert\arctan t\right\rvert
        \end{align*}
    \end{proof}

\begin{claim}\label{clm:grass_dist_other}
    Let $Q:=\left(\begin{array}{c|c}
        \mathbf{0}_{n-1,1} & I_{n-1} \\
        \hline
        1 & \mathbf{0}_{1,n-1}
    \end{array}\right)$ be the permutation matrix which moves up the indices of row vectors. Further, let $U,V_t$ be as in Claim~\ref{clm:grass_dist_0}, and let $U_\varnothing=QU$. The Grassmann distance between $\mathbf{col}\left(V_t\right)$ and $\mathbf{col}U_\varnothing$ is equal to $\lvert\arccot t\rvert$.
\end{claim}

\begingroup
\allowdisplaybreaks
\begin{proof}
        Following the discussion of Jordan angles in \cite{neretin}, we get:\\
        
        \begin{align*}\mathbf{dist}_{\mathbf{Gr}}\left(\mathbf{col}V_t,\mathbf{col}U_\varnothing\right)&=\left(\mathbf{Tr}\left[\arccos\left(\left[\left(U_\varnothing^\top U_\varnothing\right)^{-1}U_\varnothing^\top V_t\left(V_t^\top V_t\right)^{-1}V_t^\top U_\varnothing\right]^{1/2}\right)^2\right]\right)^{1/2} \\
            &=\Bigg(\mathbf{Tr}\Bigg[\arccos\bigg(\bigg[\left(I_k\mid\mathbf{0}_{k,n-k}\right)Q^\top\left(\frac{I_k}{t\vec{e}_{n-k}\vec{e}_k^\top}\right)\left(I_k-\frac{t^2}{1+t^2}\vec{e}_k\vec{e}_k^\top\right) \\
            &\hspace{1cm}\cdot\left(I_k\mid t\vec{e}_k\vec{e}_{n-k}^\top\right)Q\left(\frac{I_k}{\mathbf{0}_{n-k,k}}\right)\bigg]^{1/2}\bigg)^2\Bigg]\Bigg)^{1/2} \\
            &=\Bigg(\mathbf{Tr}\Bigg[\arccos\bigg(\bigg[\left(\begin{array}{c|c}
                \mathbf{0}_{1,k-1} & t \\
                \hline
                I_{k-1} & \mathbf{0}_{k-1,1}
            \end{array}\right)\left(I_k-\frac{t^2}{1+t^2}\vec{e}_k\vec{e}_k^\top\right) \\
            &\hspace{1cm}\cdot\left(\begin{array}{c|c}
                \mathbf{0}_{k-1,1} & I_{k-1} \\
                \hline
                t & \mathbf{0}_{1,k-1}
            \end{array}\right)\bigg]^{1/2}\bigg)^2\Bigg]\Bigg)^{1/2} \\
            &=\left(\mathbf{Tr}\left[\arccos\left(\left[\left(\begin{array}{c|c}
                t^2 & \mathbf{0}_{1,k-1} \\
                \hline
                \mathbf{0}_{k-1,1} & I_{k-1}
            \end{array}\right)-\frac{t^4}{1+t^2}\vec{e}_1\vec{e}_1^\top\right]^{1/2}\right)^2\right]\right)^{1/2} \\
            &=\left(\mathbf{Tr}\left[\arccos\left(\left(\begin{array}{c|c}
                \frac{t^2}{1+t^2} & \mathbf{0}_{1,k-1} \\
                \hline
                \mathbf{0}_{k-1,1} & I_{k-1}
            \end{array}\right)^{1/2}\right)^2\right]\right)^{1/2} \\
            &=\left(\mathbf{Tr}\left[\left(\begin{array}{c|c}
                \arccos\left(\sqrt{\frac{t^2}{1+t^2}}\right) & \mathbf{0}_{1,k-1} \\
                \hline
                \mathbf{0}_{k-1,1} & \mathbf{0}_{k-1,k-1}
            \end{array}\right)^2\right]\right)^{1/2} \\
            &=\lvert\arccot t\rvert
        \end{align*}
    \end{proof}
\endgroup

\begin{proof}(of Claim~\ref{clm:Q_A}) 
\label{prf:Q_A}
        We first show that $Q_A$ is orthogonal. Observe that $Q_AQ_A^\top$ is equal to
        \begin{align*}
            \left(\begin{array}{c|c}
                \left(I_k+A^\top A\right)^{-1}+A^\top\left(I_{n-k}+AA^\top\right)^{-1}A & \left(I_k+A^\top A\right)^{-1}A^\top-A^\top\left(I_{n-k}+AA^\top\right)^{-1} \\
                \hline
                A\left(I_k+A^\top A\right)^{-1}-\left(I_{n-k}+AA^\top\right)^{-1}A & A\left(I_k+A^\top A\right)^{-1}A^\top+\left(I_{n-k}+AA^\top\right)^{-1}
            \end{array}\right).
        \end{align*}
        To show $Q_A$ is orthogonal, then, it suffices to show that:
        \begin{enumerate}[label=\arabic*)]
            \setlength\itemsep{0em}
            \item $\left(I_k+A^\top A\right)^{-1}A^\top-A^\top\left(I_{n-k}+AA^\top\right)^{-1}=\mathbf{0}_{k,n-k}$;
            \item $\left(I_k+A^\top A\right)^{-1}+A^\top\left(I_{n-k}+AA^\top\right)^{-1}A=I_k$; and
            \item $A\left(I_k+A^\top A\right)^{-1}A^\top+\left(I_{n-k}+AA^\top\right)^{-1}=I_{n-k}$.
        \end{enumerate}
        These are shown in Lemmas, \ref{lem:Q_A_lem_2}, \ref{lem:Q_A_lem_1}, and \ref{lem:Q_A_lem_3}, respectively. \\
        
        It remains to show that $Q_A\left(\begin{array}{c|c}
            I_k & \mathbf{0}_{k,n-k} \\
            \hline
            \mathbf{0}_{n-k,k} & \mathbf{0}_{n-k,n-k}
        \end{array}\right)Q_A^\top=\tilde{\varphi}_0(A)$. By definition of $\varphi_0$, 
        \begin{equation*}
            \varphi_0(A)=\left(\begin{array}{c|c}
                \left(I+A^\top A\right)^{-1} & \left(I+A^\top A\right)^{-1}A^\top \\
                \hline
                A\left(I+A^\top A\right)^{-1} & A\left(I+A^\top A\right)^{-1}A^\top
            \end{array}\right),
        \end{equation*}
        and so completing the proof is a straightforward computation.
\end{proof}

\begin{lemma}\label{lem:Q_A_lem_2}
    \begin{equation*}
        \left(I_k+A^\top A\right)^{-1}A^\top-A^\top\left(I_{n-k}+AA^\top\right)^{-1}=\mathbf{0}_{k,n-k}
    \end{equation*}
    \end{lemma}
    \begin{proof}
        Using a Neumann series representation of the matrix inverse, e.g.,~\cite{stewart}, the lemma is proved by the following series of deductions.
        \begin{align*}
            A^\top\left(I_{n-k}+AA^\top\right)^{-1}&=A^\top\sum_{j=0}^\infty(-1)^j\left(AA^\top\right)^j \\
            &=\left(\sum_{j=0}^\infty(-1)^j\left(A^\top A\right)^j\right)A^\top \\
            &=\left(I_k+A^\top A\right)^{-1}A^\top
        \end{align*}
    \end{proof}

\begin{lemma}\label{lem:Q_A_lem_1}
    \begin{equation*}
        \left(I_k+A^\top A\right)^{-1}+A^\top\left(I_{n-k}+AA^\top\right)^{-1}A=I_k
    \end{equation*}
    \end{lemma}
    \begin{proof}
        By Lemma \ref{lem:Q_A_lem_2}, the following series of deductions holds.
        \begin{align*}
            \left(I_k+A^\top A\right)^{-1}+A^\top\left(I_{n-k}+AA^\top\right)^{-1}A&=\left(I_k+A^\top A\right)^{-1}+\left(I_k+A^\top A\right)^{-1}A^\top A \\
            &=\left(I_k+A^\top A\right)\left(I_k+A^\top A\right)^{-1} \\
            &=I_k
        \end{align*}
    \end{proof}

\begin{lemma}\label{lem:Q_A_lem_3}
    \begin{equation*}
        A\left(I_k+A^\top A\right)^{-1}A^\top+\left(I_{n-k}+AA^\top\right)^{-1}=I_{n-k}
    \end{equation*}
    \begin{proof}
        A proof for this Lemma is easily recreated from the method used to prove Lemma \ref{lem:Q_A_lem_1}.
    \end{proof}
\end{lemma}

\begin{lemma}\label{len:A_commutes_with_root}
    $$\begin{array}{c}
        A\sqrt{I_k+A^\top A}=\sqrt{I_{n-k}+AA^\top}A, \\
        A^\top\sqrt{I_{n-k}+AA^\top}=\sqrt{I_k+A^\top A}A^\top
    \end{array}$$
    \begin{proof}
        This proof relies on a Neumann series representation of the square root of a matrix, e.g.,~\cite{stewart}.
        \begin{align*}
            A\sqrt{I_k+A^\top A}&=A\left(I_k-\sum_{j=1}^\infty\left\lvert\binom{1/2}{j}\right\rvert\left(I_k-I_k-A^\top A\right)^j\right) \\
            &=A\left(I_k-\sum_{j=1}^\infty\left\lvert\binom{1/2}{j}\right\rvert\left(-A^\top A\right)^j\right) \\
            &=A-\sum_{j=1}^\infty\left\lvert\binom{1/2}{j}\right\rvert A\left(-A^\top A\right)^j \\
            &=A-\sum_{j=1}^\infty\left\lvert\binom{1/2}{j}\right\rvert\left(-AA^\top\right)^jA \\
            &=\left(I_{n-k}-\sum_{j=1}^\infty\left\lvert\binom{1/2}{j}\right\rvert\left(-AA^\top\right)^j\right)A \\
            &=\sqrt{I_{n-k}+AA^\top}A
        \end{align*}
        The remainder of the claim is proved by replacing $I_k$ with $I_n$ and $A$ with $A^\top$.
    \end{proof}
\end{lemma}

\begin{claim}\label{clm:grass_exp_approx}
    Let $\mathfrak{g}$ be the Riemannian metric on $\mathbf{Gr}_{n,k}$ inherited from the Euclidean metric in $\mathbb{R}^{n \times n}$. The retraction on \atlasgrass objects approximats geodesics $\gamma:[0, 1] \to \mathbf{Gr}_{n,k}$ with initial conditions $\gamma(0) = \varphi(\vec{0}), \dot{\gamma}(0) = \vec{\tau}$ with error of order $O\left(\left\lVert \vec{\tau} \right\rVert_{\mathfrak{g}}^3 \right)$.
    \begin{proof}
    
        From Claim~\ref{clm:grass_dist_0}, we know that for charts $\varphi: \mathcal{V} \to \mathcal{U}$ in \atlasgrass, lengths of paths constrained to coordinate axes in $\mathcal{V}$ are preserved by the automorphism
        \begin{align*}
            \alpha: \mathcal{V} &\to \mathcal{V} \\
            \vec{\xi} &\mapsto A \vec{\xi},
        \end{align*}
        where $A$ is a diagonal matrix whose nonzero entries are in $\{-1, 1\}$. For this reason, the partial derivatives $g_{\mu \mu, \nu}$ are equal to their additive inverses at $\vec{\xi} = \vec{0}$, meaning they must be zero. By the parallelogram law, the partial derivatives $g_{\mu \nu, \lambda}$ must also vanish at $\vec{\xi} = \vec{0}$. Because these partial derivatives all vanish, the Christoffel symbols $\Gamma^\lambda_{\mu \nu}$ must also vanish at $\vec{\xi} = \vec{0}$. 
        
        We can exploit this fact by considering the third-order Taylor expansion of the geodesic $\gamma$:
        \begin{equation}\label{eq:ret_taylor_approx}
            \gamma(t) = \gamma(0) + t \dot{\gamma}(0) + \frac{t^2}{2} \ddot{\gamma}(0) + \frac{t^3}{6} \dddot{\gamma}(0) + O(t^4).
        \end{equation}
        The constant and first-order terms are given by the initial conditions $\gamma(0) = \varphi(\vec{0})$ and $\dot{\gamma}(0) = \vec{\tau}$. The second-order term is given by the geodesic equation
        \begin{equation}
            \ddot{\gamma}^i + \Gamma^i_{jk} \dot{\gamma}^j \dot{\gamma}^k = 0.
        \end{equation}
        As we know that $\Gamma^i_{jk}$ vanishes at $t = 0$, and so $\ddot{\gamma}(0) = \vec{0}$.
        So, the approximation error of the \atlasgrass retraction at $\vec{\xi} = \vec{0}$ must have $O\left(\left\lVert \vec{\tau} \right\rVert_{\mathfrak{g}}^3 \right)$ approximation error.
    \end{proof}
\end{claim}

\section{\atlas methods applied to the $k_0$ Klein bottle} \label{app:klein}
\begin{figure}[h]
    \centering
    \centerline{\noindent}
    \begin{overpic}[width=0.29\columnwidth, trim=48 24 34 54, clip]{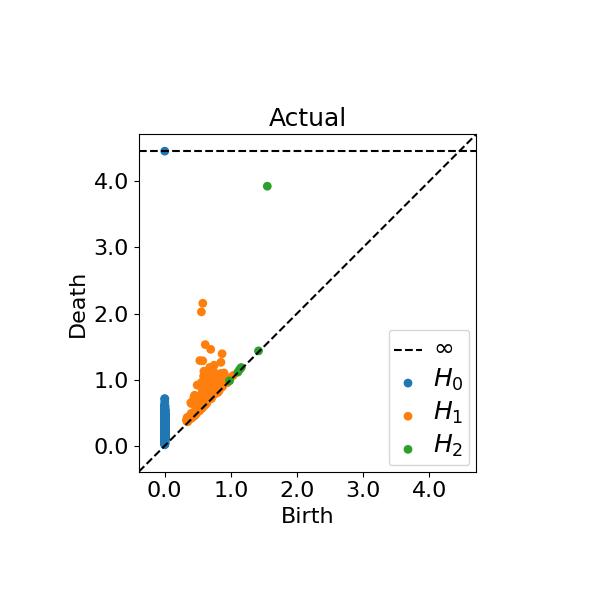}\put(5,95){\textbf{A}}\end{overpic}
    \begin{overpic}[width=0.29\columnwidth, trim=48 24 34 54, clip]{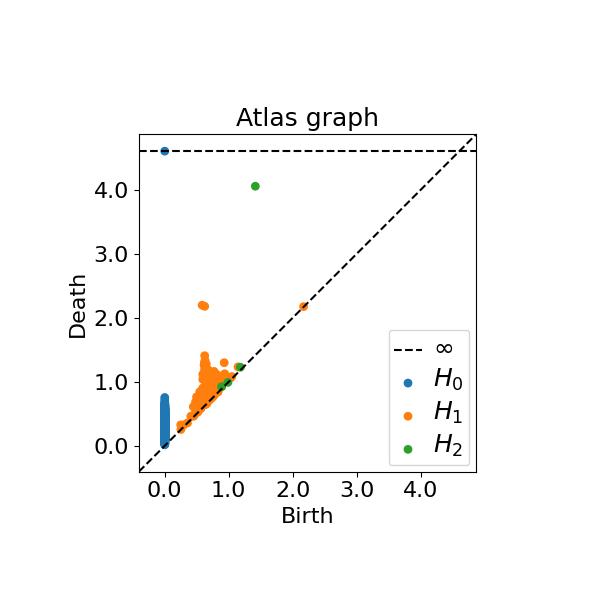}\put(5,95){\textbf{B}}\end{overpic}
    \begin{overpic}[width=0.3\columnwidth, trim=20 50 100 45, clip]{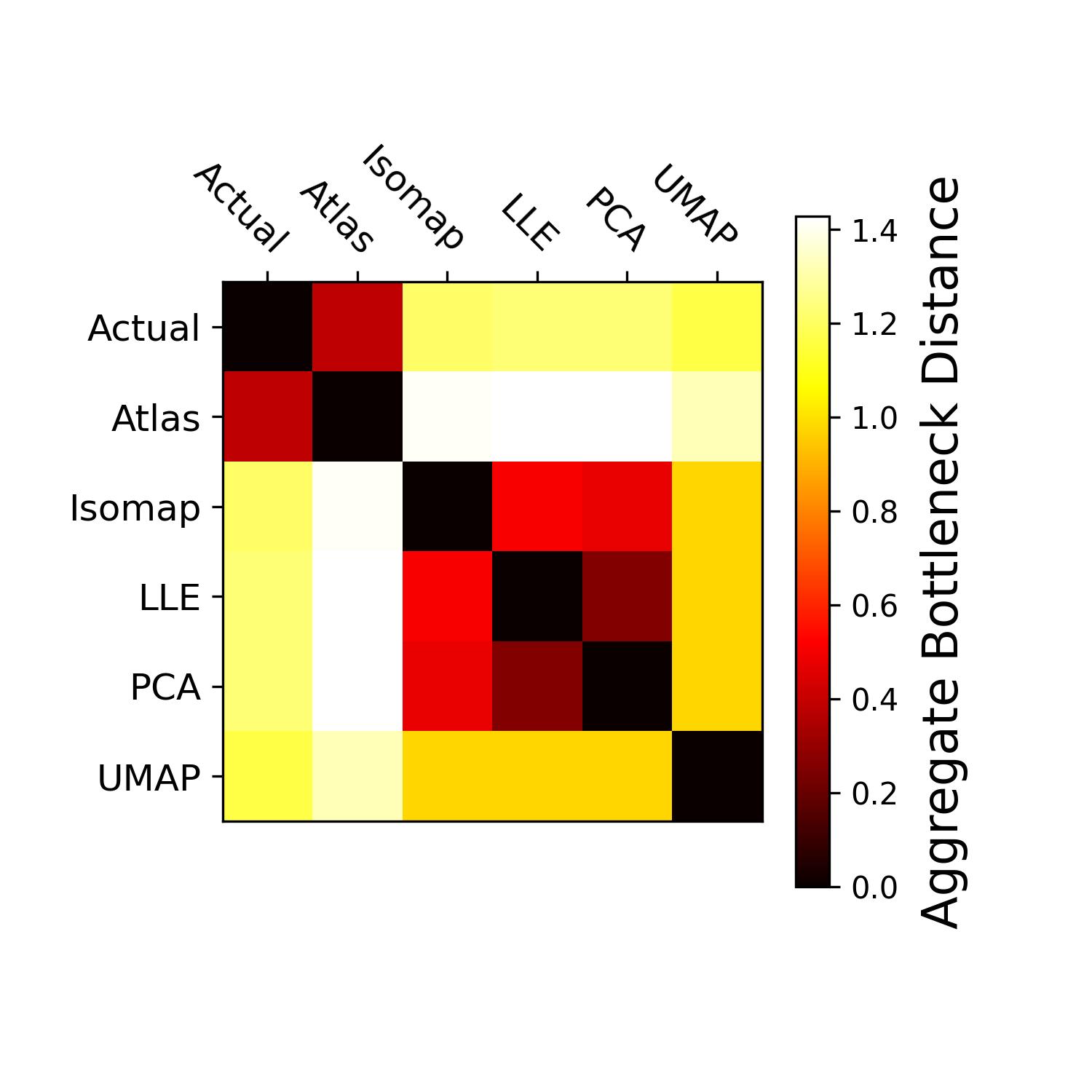}\put(5,95){\textbf{C}}\end{overpic}
    \includegraphics[width=0.1\columnwidth, trim=260 0 0 0, clip]{graphics/klein/bottleneck_heatmap.jpg}
    
    \caption{\textbf{Our \atlas representation of the high-contrast image patches manifold preserves its H1 and H2 homology groups.} \textbf{A)} Representations of the H0, H1, and H2 homology groups of the $k_0$ Klein bottle computed using Vietoris-Rips persistent homology on 640 points sampled uniformly in polar coordinates and mapped to the $k_0$ geometry. \textbf{B)} Representations of the H0, H1, and H2 homology groups of an 64-chart \atlaslearn instance computed using Vietoris-Rips persistent homology on 640 points, sampled by taking 10 points from each chart using the method \texttt{atlas\_grid\_sample}. \textbf{C)} The aggregate bottleneck distance between Vietoris-Rips persistence diagrams generated by different 2D representions using 640 sampled points}\label{fig:h2_homology}
\end{figure}

\subsection{Persistence diagram computation (Fig. \ref{fig:h2_homology}A,B)}\label{sec:persistent_homology}
We compute persistent homology on point clouds generated from our \atlaslearn representation of high-contrast image patches, as well as on points sampled from a parametrization from Sritharan \textit{et al.} \cite{sritharan} of the same manifold (Fig.~\ref{fig:h2_homology}).

Point samples from \atlas were generated by performing the following process for each of the 64 charts $(\mathcal{V},\varphi)$ in our \atlas representation:
\begin{itemize}
    \setlength\itemsep{0em}
    \item For chart $\mathcal{V}$ of radius $r$, a point $\tv$ is sampled uniformly from the ball of radius $r$ in $\mathcal{V}$;
    \item $\tv$ is rejected if it does not belong to the same coordinate chart;
    \item if $\tv$ is not rejected, return $\tilde{\varphi}^{-1}\left(\tv\right)\in\mathbb{R}^9$.  
\end{itemize}

For the Sritharan parametrization of the Klein bottle, a point is sampled by sampling coordinates $(\theta,\phi)$ from the uniform measure on $[0,\pi]\times [0,2\pi]$, mapping the coordinates into $\mathbb{R}^{3\times 3}$ by the restriction of the map $k_{\theta,\phi}$ to $\left\{-1,0,1\right\}\times\left\{-1,0,1\right\}$, and reshaping the result as an element of $\mathbb{R}^9$. 

Ten points were sampled for each of the 64 charts in the \atlas, for a total of 640 points, and $640$ points were sampled from Sritharan's parametrization of the Klein bottle. 

Persistent homology was computed on the samples by Vietoris-Rips complex using the Python package Ripser~\cite{ripser}. 


\subsection{Pairwise aggregate bottleneck distance computation (Fig.~\ref{fig:h2_homology}C)}\label{app:bottleneck_heatmap}
Let $P_0,P_1,P_2$ be persistence diagrams generated from data $X$ corresponding to $H_0$, $H_1$, and $H_2$ features, respectively. Let $Q_0,Q_1,Q_2$ be generated similarly from data $Y$. We define the \textit{aggregate bottleneck distance} between diagram triplets $\left(P_0,P_1,P_2\right)$ and $\left(Q_0,Q_1,Q_2\right)$ as $\sqrt{\sum_{i=0}^2d^2_\text{B}\left(P_i,Q_i\right)}$, where $d_\text{B}$ denotes the bottleneck distance, e.g.,~\cite{oudot}. This distance function 
reflects the bottleneck distance between persistence diagrams in each dimension, making it a natural way to measure the preservation of $H_0$, $H_1$, and $H_2$ features by different dimensionality reduction measures.

To create the heatmap in Fig.~\ref{fig:h2_homology}C, the same points were used as described in Section~\ref{sec:persistent_homology}. PCA (restricted to the top two or five principal components) and a two-dimensional UMAP were each computed on the points sampled from the Sritharan parametrization. We excluded $t$-SNE from the bottleneck distance computations, due to its very high aggregate bottleneck distances to all other methods.

\subsection{Geodesic distance computations (Fig.~\ref{fig:dist_fig})}\label{sec:geodesic_distances}

For the experiment comparing geodesic distances, we approximated true geodesic distance as follows. First, we generated sample points from the Sritharan parametrization by mapping a $1000\times 1000$ grid of evenly spaced points in the set $[0,\pi]\times[0,2\pi]$ into $\mathbb{R}^{3\times3}$ through the restriction of $k_0$ to $\{-1,0,1\} \times \{-1,0,1\}$. For \texttt{Isomap} specifically, we uniformly subsampled these points down to 20,000 in order to keep memory requirements under 4GB. These points in $\mathbb{R}^{3 \times 3}$ were then reshaped these as vectors in $\mathbb{R}^9$. These points were used to create a $15$-nearest-neighbors graph, with each edge weighted by the Euclidean distance between its endpoints. True geodesic distance between a given pair of points was then approximated by computing a shortest path in this graph. 

For \texttt{Isomap}, \texttt{PCA}, and \texttt{$t$-SNE}, geodesic distances were computed via an analogous approach. That is, each of these transformations was used to map the sampled grid points into a new, separate representation, on which $15$-nearest-neighbor graphs with distance-weighted edges were computed. Note that these graphs have one million nodes, while \atlas uses a graph $G_{\epsilon, \delta}$ of only 28,700 nodes for the values of $\epsilon=0.6$ and $\delta=0.1$ (Sec.~\ref{sec:app_naive_dist}); therefore, naively, one might expect geodesic distances to be better preserved by \texttt{PCA}, $t$-\texttt{SNE}, and \texttt{Isomap} in the transformed space than by the \atlas representation. 

To investigate how well geodesic distances were preserved by \atlas, PCA, $t$-SNE, and UMAP, 100 pairs of points were randomly sampled without replacement from the one million grid points in the Sritharan parametrization. For \atlas, these 100 pairs of points were ingested into the \atlas, and the geodesic distance between each pair of points was approximated as 
the na\"ive approximate distance between the points (Sec.~\ref{sec:app_naive_dist}).

\begin{figure}[t]
    \centering
    \includegraphics[width=0.45\linewidth]{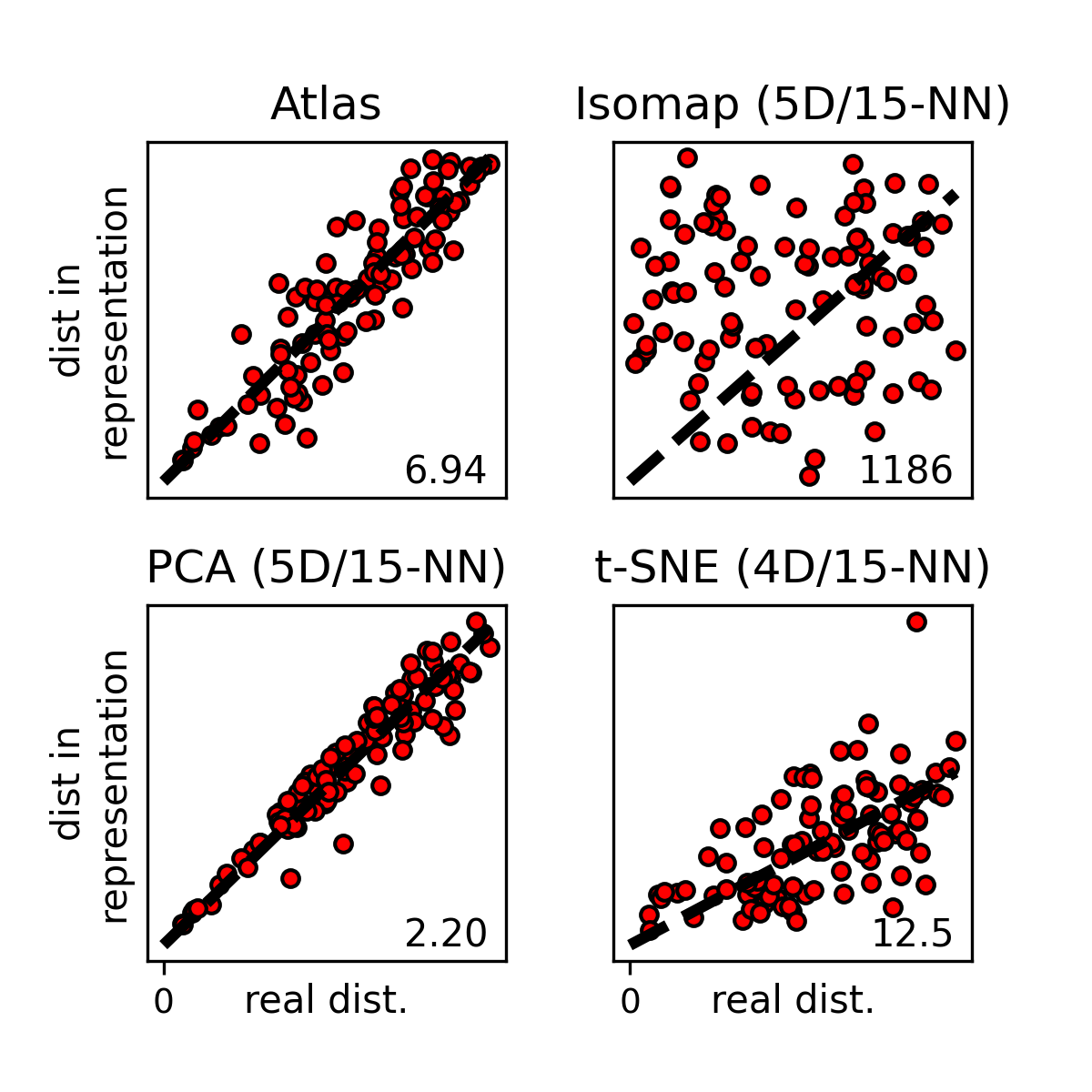} \hspace{1cm}\includegraphics[width=0.45\linewidth]{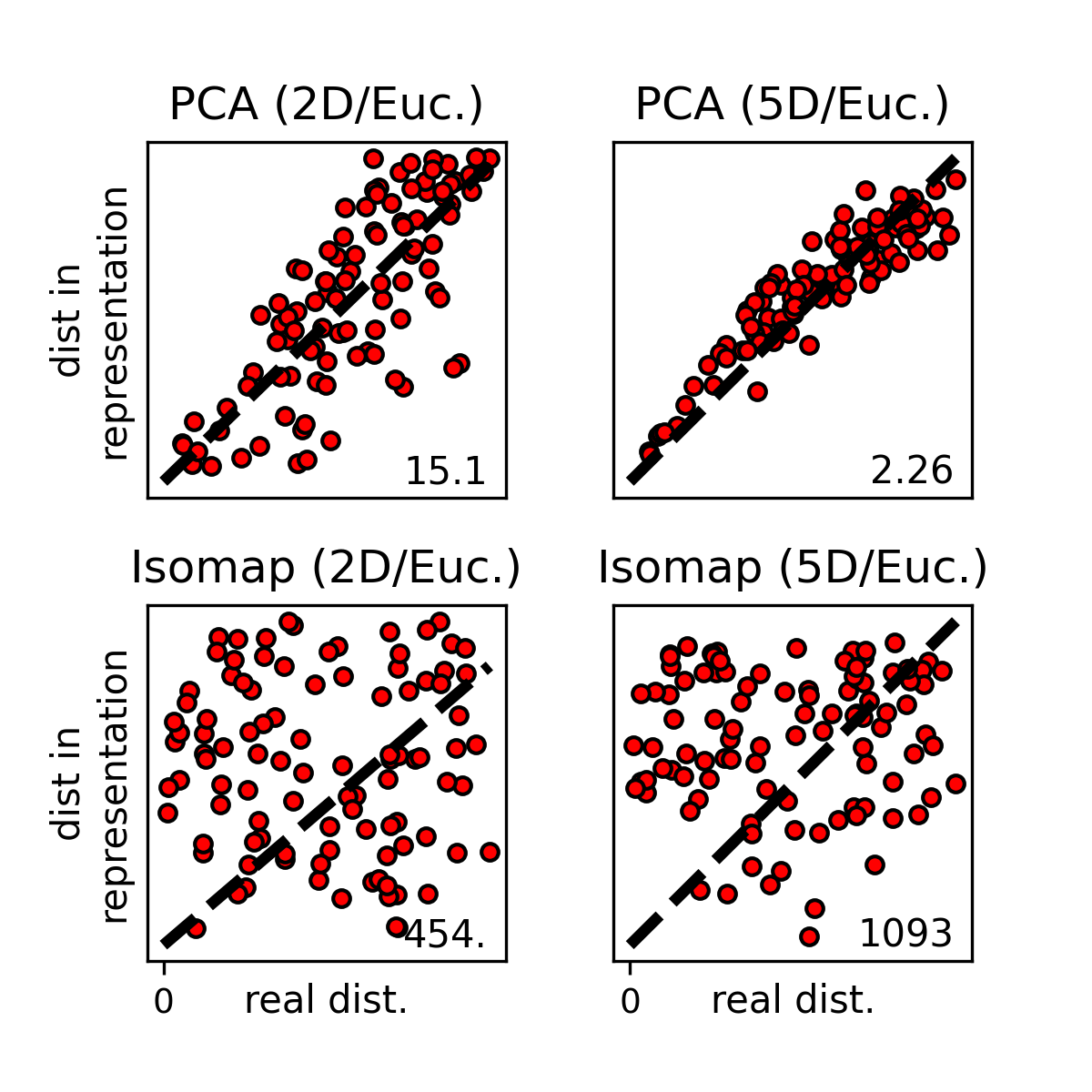}
    \caption{\textbf{Even at the maximum embedding dimension from the Whitney Embedding Theorem, state-of-the-art, nonlinear dimensionality reduction techniques fail to preserve geodesic distances.} \texttt{Isomap}, \texttt{PCA} and \texttt{$t$-SNE} were computed on a mesh of 1 million points. \atlaslearn was asked to produce 64 charts, given only a 10,000-points mesh. For 100 randomly sampled pairs of points on the Klein bottle, each scatter plot shows each pair (dot) according to a precise estimate of their true geodesic distance (x axis, ``real dist.'') versus the embedding distance (y axis, ``dist. in representation'') for each embedding (panel title). \atlas is a 2D representation. For all other methods, we specify parameters in the form $(a/b)$ where $a$ indicates the dimensionality of the representation used and $b$ the method used for computing geodesics. For $b$, the options are either \textit{15-NN}, indicating that shortest path were computed on a 15-nearest-neighbor graph, or \textit{Euc.}, meaning that distances were estimated by the Euclidean distance in the represented embedding. %
    Metric distortion is given in bottom right corner. Dashed line is fit with $y$-intercept zero.}
    \label{fig:dist_whitney}
\end{figure}

\subsubsection{State-of-the-art dimensionality reduction techniques fail to preserve geodesic distances even at the Whitney dimension}

At least four dimmensions are needed to embed the Klein bottle into Euclidean space.
In Figure~\ref{fig:dist_whitney}, we show that \texttt{PCA} finds a five-dimensional linear subspace onto which data sampled from the $k_0$ Klein bottle can be preserved with little metric distortion, whereas \texttt{Isomap} and \texttt{$t$-SNE} fail to preserve the metric structure even when having enough dimensions. The methodology for sampling points, performing dimensionality reduction, and computing geodesic distances is as described in Section~\ref{fig:dist_fig}. For \texttt{$t$-SNE}, four dimensions were used instead of five due to runtime constraints.

\subsection{Packages used for non-atlas dimensionality reduction}\label{app:non_atlas_packages}
PCA and $t$-SNE were computed using scikit-learn (version 1.3.2)~\cite{scikit-learn}
. We ran $t$-SNE was  with \texttt{perplexity=5.0} and otherwise default arguments.

\section{Computing Riemannian principal boundaries}\label{app:rpb}

\subsection{Parameter choices and modifications}\label{sec:params_and_mods}
We implemented the RPB algorithm as described in~\cite{yao_2020}, with the exceptions described here. 

In Yao, \textit{et al.}~\cite{yao_2020}, Sec.~2.2, a univariate kernel $\kappa_h$ is used to define which points are included in the computation of the local covariance matrix $\Sigma_h$. For this kernel, we use the indicator function for the ball of radius $h$.

The original algorithm computes a weighted average of the first derivatives of the two principal flows, where $\lambda_\delta(t)$ is the weight at iteration $t$ (\cite{yao_2020}, Eq.~10). In our implementation, we assume that $\lambda_\delta(t)=1/2$ for all $t$. This choice was motivated by both simplicity and practical considerations, i.e., it helped avoid issues where the boundary could collapse into one of the principal flows. 

The differential equation for a principal flow $\gamma^+$ (induced by Equation~5 of \cite{yao_2020}) effectively follows the vector field defined by the top eigenvector of the local covariance matrix $\Sigma_h$. To avoid oscillations in the principal flow direction, due either to the insensitivity of eigenvector computations to multiplication by $-1$ or non-smooth changes in the top eigenvectors, we enforce a positive inner product between tangent vectors in adjacent iterates. 
Further,
if the support of the data is sufficiently sparse, in practice, the top eigenvector field for the principal flow will be dominated by noise and will eventually cause the boundary curve to move away from the data, which also causes $\Sigma_h$ to become undefined. 
To prevent this from happening, we use the following modification to correct the principal flow solution by moving it towards the mean of the local data:

Let $W$ be the top eigenvector field of the local covariance matrix $\Sigma_h$ (as in Equation~4 of \cite{yao_2020}). Instead of following the update rule $\dot{\gamma}=W(\gamma),$ we instead follow the rule
\begin{equation}\label{eqn:modified_yao}
    \dot{\gamma}=W(\gamma)+\alpha\left(I-W(\gamma)W(\gamma)^\top\right)\left(\frac{1}{\sum_i\kappa_h\left(x_i,\gamma\right)}\sum_i\log_\gamma x_j\right),
\end{equation}
where $\alpha>0$ is a correction factor that moves $\dot{\gamma}$ toward the mean Riemannian logarithm of nearby sample points, projected onto the orthocomplement $I-W(\gamma)W(\gamma)^\top$ of the top eigenvector of the local covariance matrix

\subsubsection{State-of-the-art dimensionality reduction methods fail to learn intelligible separator between convex and concave patches in two dimensions}\label{app:bad_separator_klein}

\begin{figure}[t]
    \centering
    \includegraphics[width=0.85\columnwidth]{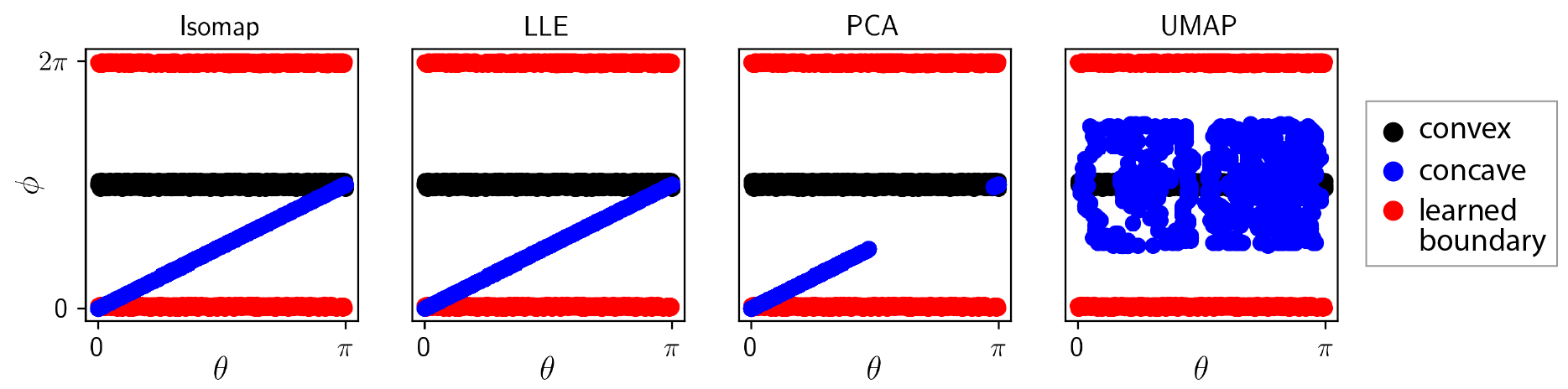} \\
    \caption{\textbf{Separating boundaries between convex and concave patches learned within 2D representations from state-of-the-art dimensionality reduction techniques forfeit the topology of the $k_0$ Klein bottle.} Dimensionality reduction techniques are applied to 20,000 points (Isomap) or one million points (LLE, PCA, UMAP) from the $k_0$ Klein bottle in order to learn 2D representations of the data. Then, $1,000$ convex (black) patches and $1,000$ concave (red) patches are mapped into the learned representation. Linear SVM is applied in order to learn a separator between convex and concave patches. The $1,000$ points closest to the boundary in the learned space are then visualized in the polar representation of the Klein bottle (blue).}
    \label{fig:bad_separator_klein}
\end{figure}

To highlight the significance of the RPB algorithm learning an interpretable, intrinsic separator between convex and concave patches on a 2D \atlas object in Section~\ref{app:klein}, we compare against other 2D representations learned using state-of-the-art dimensionality reduction techniques. Specifically, we take 1 million points from a grid in the polar representation of the Klein bottle, and subsample $20,000$ for \texttt{Isomap}, as in Section~\ref{sec:geodesic_distances}. After transforming these polar points into $\mathbb{R}^9$, we apply \texttt{Isomap}, \texttt{$t$-SNE}, and \texttt{PCA} as described in the same section, and additionally apply \texttt{UMAP} from the \texttt{scikit\_learn} implementation using \texttt{n\_neighbors=5} and othwerwise default parameters. We then map $1,000$ convex and $1,000$ concave patches into the learned representations. Using the linear SVM classifier implementation from \texttt{scikit\_learn}, we learn linear boundaries separating convex and concave patches in each representation~\cite{scikit-learn}. Recall that the RPB algorithm is intended to be a Riemannian-geometric generalization of linear SVM. This is the closest analog of the RPB algorithm in these learned, one-chart representations of the data, and is chosen for the sake of fair comparison~\cite{yao_2020}.

When the boundaries learned in the 2D representations are mapped back into the ambient space---by identifying the preimages of the $1,000$ points closest to the learned boundary in each represenation---we see that \texttt{Isomap}, \texttt{PCA}, and \texttt{$t$-SNE} capture a local separator between the convex and concave patches. However, these local separators take into account neither the compactness nor the unorientability of the Klein bottle.

\section{RNA Experiments}\label{app:rna}

\subsection{Using \atlaslearn to create 5-dimensional representation of  hematopoietic single-cell transcriptomic data}\label{app:atlas_learn_dynamo}

\begin{figure}[h]
    \centering
    \includegraphics[width=\columnwidth, trim=2cm 0 2cm 0, clip]{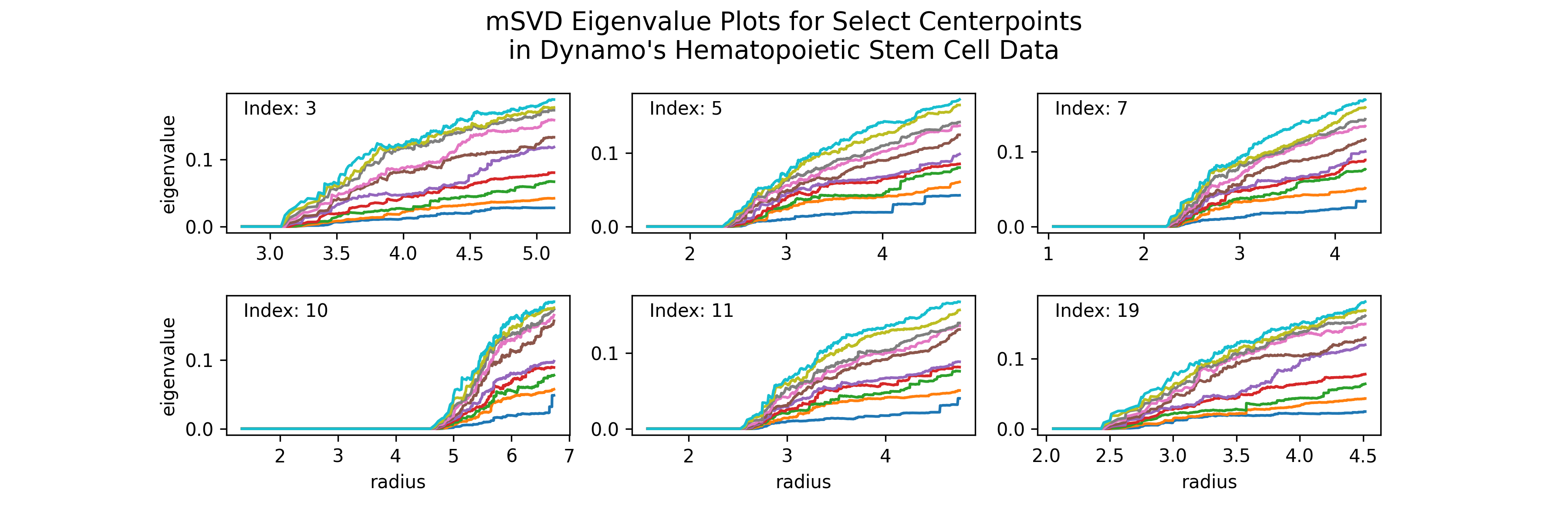}
    \caption{\textbf{mSVD eigenvalue plots based at chosen centerpoints on the Dynamo HSC data in 30-PC space suggest the data lie close to a submanifold of dimension five.} For each centerpoint $\vec{x}$, all other data were sorted by distance from $\vec{x}$. At each radius $r$, the eigenvalues of the covariance matrix $\Sigma$ containing points within distance $r$ of $\vec{x}$ are plotted. This is done for the 1200 points closest to $\vec{x}$. These heuristic, graphical adaptations of the mSVD technique from~\cite{little} demonstrate a prominent gap between the fifth and sixth eigenvalues.}
    \label{fig:dynamo_msvd}
\end{figure}

Single-cell RNA-sequencing data (with metabolically labeled new RNA tagging) from hematopoietic cells were retrieved using the  \texttt{dynamo} Python package, following the ``scNT-seq human hematopoiesis dynamics'' tutorial notebook, which creates a data representation based on the top 30 principal components (PCs)~\cite{dynamo_github, qiuMappingTranscriptomicVector2022}. Our preliminary,  multiscale-singular vector decomposition (mSVD) analysis of these data suggests that, at several points of the manifold, the data are well approximated by a five-dimensional submanifold (Figure~\ref{fig:dynamo_msvd}). An \atlas data structure is learned on the 30-PC representation using \atlaslearn (Algorithm~\ref{alg:learn_atlas_graph}) with dimensionality 5 and 30 coordinate charts. For several random seeds, there are enough points in each randomly assigned chart for the tangent Stiefel matrix $L$ to be learned, but not enough for the quadratic coefficients in $\hat{\mathbf{h}}$ to be determined. For this reason, we set the $M_i$ an arbitrary, Stiefel orthocomplement of $L_i$ and all quadratic coefficients in $K_i$ to be zero.

\subsection{Computation, integration, and interpretation of an RNA velocity vector field}

\begin{figure}[h]
    \centering
    \includegraphics[width=1.0\columnwidth]{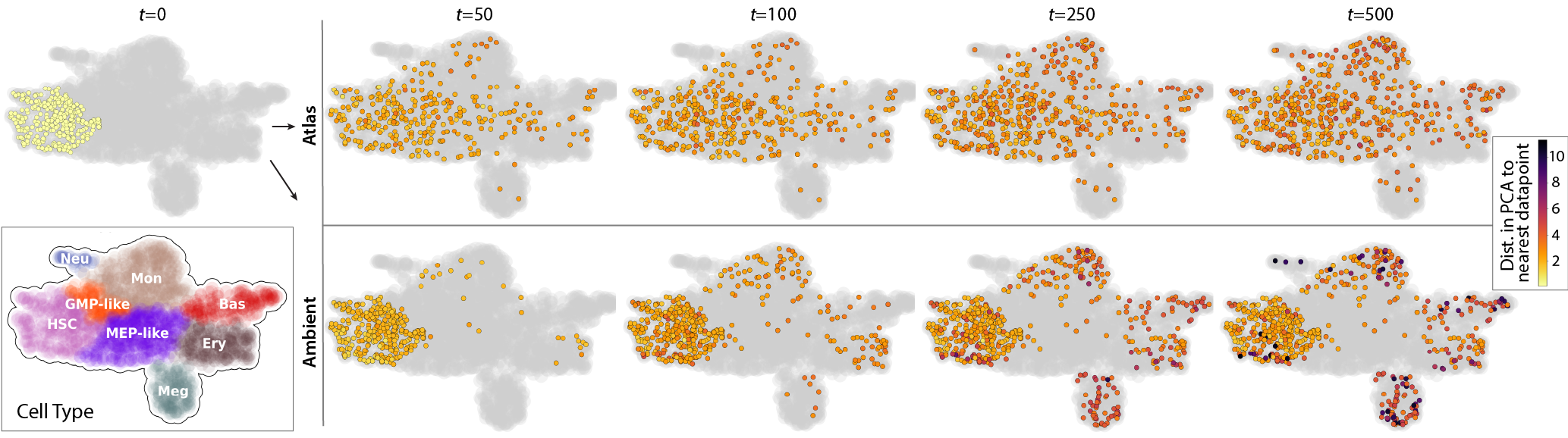}
    \caption{\textbf{Integrating the vector field in Section~\ref{sec:bio} in the tangent space of an \atlas object preserves proximity to the data while capturing all cell types present.} Each plot shows a UMAP representation of single-cell RNA-sequencing data (with metabolically labeled new RNA tagging)  from 1,947 hematopoietic cells (colored by cell type, bottom left, or in gray)~\cite{qiuMappingTranscriptomicVector2022}. The RNA velocity vector field $\vec{V}_{\texttt{amb}}$ was integrated as an ODE both in the tangent plane of an \atlaslearn structure (``Atlas'', top row) and in 30-dimensional PCA space (``Ambient'', bottom row). For 309 initial conditions ($t=0$, top left, in yellow), each corresponding to a HSC, integration was performed for 500 time steps $t$ in each space. At each labelled time step $t\in \{50,100,250,500\}$ (panels), the iterate is visualized in the UMAP by overplotting the datapoint that is closest to the iterate in color, representing its Euclidean distance from the iterate in 30-PC space. Abbreviations: HSC: hematopoietic stem cell; GMP: granulocyte-monocyte progenitor; MEP: megakaryocyte-erythroid progenitor; Neu: neutrophil; Mon: monocyte; Bas: basophil; Ery: erythroid; Meg: megakaryocyte. 
    }
    \label{fig:dynamo_speckles}
\end{figure}

We isolate the 309 points labeled as hematopoietic stem cells (HSCs), treating each as the initial condition to an ODE. We consider two ODEs: the one induced by the \texttt{Dynamo} vector field $\vec{V}_{\texttt{amb}}$ in the ambient space and the vector field $\vec{V}_{\texttt{atlas}}$ induced by projecting $\vec{V}_{\texttt{amb}}$ onto the learned tangent bundle. 
The $\vec{V}_{\texttt{amb}}$ ODE is integrated in the ambient space by the 4(5) Runge-Kutta method using the \texttt{Dynamo} function \texttt{dyn.vf.VectorField.integrate}, with arguments \texttt{interpolation\_num=500} and \texttt{t\_end=t\_term}\footnote{Here, \texttt{t\_term = 2 * dynamo.tools.utils.getTend(adata\_hsc.obsm["X\_pca"], adata\_hsc.obsm["velocity\_pca"])}.}. The $\vec{V}_{\texttt{atlas}}$ is integrated by the forward-Euler method with stepsize 1 for 500 iterations, using the method $\mathbf{Ret}_\cdot \cdot$ in Algorithm~\ref{spec:atlas_graph} as the update step. Integrating $\vec{V}_{\texttt{amb}}$ results in iterates $\vec{x}_i^\texttt{amb} \in \mathbb{R}^{30}$ for $i\in\{0, \ldots, 500\}$. Similarly, integrating $\vec{V}_{\texttt{atlas}}$ results in iterates $\left(\vec{\xi}_i, j_i\right) \in \mathbb{R}^5 \times \{1,\ldots,30\}$ for $i\in\{0, \ldots, 500\}$, where $\xi_i$ are tangential coordinates and $j_i$ are chart indices. For the sake of comparison, we define $\vec{x}_{i}^{\texttt{atlas}} = \varphi_{j_i}\left(\vec{\xi}\right)$.

The results show that  iterates $\vec{x}^{\texttt{atlas}}_i$ keep close to the manifold for all iterations, whereas iterates $\vec{x}^{\texttt{amb}}_i$ gradually depart from the manifold over time (Figure~\ref{fig:dynamo_speckles}). Moreover, iterates $\vec{x}^{\texttt{atlas}}_i$ gradually pass through canonical intermediate cell states (i.e., granulocyte-macrophage progenitor-like and megakaryocyte-erythroid progenitor-like cells) before entering terminal states (i.e., neutrophils, monocytes, basophils, erythroid cells, megakaryocytes). On the other hand, iterates $\vec{x}^{\texttt{amb}}_i$ rarely pass through intermediate states, including in time steps that are not shown. Taken together, these data suggest that using the learned atlas constrains the integration of an RNA-velocity transcriptomic vector field ODE to better reflect the data, improving the biological plausiblity of the resulting trajectories.

\putbib[references]
\end{bibunit}
\end{document}